\documentclass{article}
\usepackage{ijcai20}

\usepackage{times}

\usepackage{soul}
\usepackage{url}
\usepackage[small]{caption}
\usepackage{graphicx}
\usepackage{amsmath}
\usepackage{booktabs}
\usepackage{amsmath}

\usepackage{amssymb}
\usepackage{mathrsfs}
\usepackage{dsfont}
\usepackage{amsthm}
\usepackage[]{algorithm2e}

\usepackage{graphicx}

\newtheorem{theorem}{Theorem}
\newtheorem{lemma}[theorem]{Lemma}

\newtheorem{proposition}[theorem]{Proposition}

\newtheorem{example}[theorem]{Example}

\long\def\comment#1{}

\title{Model-theoretic Characterizations of Existential Rule Languages}

\author{Heng Zhang,$^1$ Yan Zhang,$^{2}$ Guifei Jiang\,$^3$\\
\affiliations
$^1$College of Intelligence and Computing, Tianjin University, Tianjin, China\\
$^2$School of Computing, Engineering and Mathematics, Western Sydney University, Penrith, Australia\\
$^3$College of Software, Nankai University, Tianjin, China\\
\emails
heng.zhang@tju.edu.cn, yan.zhang@westernsydney.edu.au, g.jiang@nankai.edu.cn
}

\begin{document}

\maketitle

\begin{abstract}
Existential rules, a.k.a. dependencies in databases, and Datalog+/- in knowledge representation and reasoning recently, are a family of important logical languages widely used in computer science and artificial intelligence. Towards a deep understanding of these languages in model theory, we establish model-theoretic characterizations for a number of existential rule languages such as (disjunctive) embedded dependencies, tuple-generating dependencies (TGDs), (frontier-)guarded TGDs and linear TGDs. All these characterizations hold for arbitrary structures, and most of them also work on the class of finite structures. As a natural application of these characterizations, complexity bounds for the rewritability of above languages are also identified.
\end{abstract}

\section{Introduction}

Existential rule languages, a family of languages that extend Datalog by allowing existential quantifiers in the rule head, had been initially introduced in databases in 1970s to specify the semantics of data stored in a database~\cite{AbiteboulHV1995}. Since then, existential rule languages such as tuple-generating dependencies (TGDs), embedded dependencies and equality-generating dependencies have been extensively studied. These language have been recently rediscovered as languages for data exchange~\cite{FaginKMP2005}, data integration~\cite{Lenzerini2002} and ontology-mediated query answering~\cite{CaliGLMP2010}. Towards tractable reasoning, many restricted classes of these languages have been proposed, including linear and guarded TGDs~\cite{CaliGL2012}, as well as frontier-guarded TGDs~\cite{BagetLMS11}. As a family of important logical languages, their model theory has not been fully investigated yet. In this work we aim at characterizing existential rule languages in a model-theoretic approach.

Model-theoretic characterizations, which assert that a sentence in a large language is definable in a small language if, and only if, it enjoys some semantic properties, play a key role in the study of a logical language~\cite{CK1990}. We are interested in semantic properties that can be effectively observed. Model-theoretic characterizations based on such properties thus provide a natural tool for identifying the expressivibility of a language, i.e., determining which knowledge or ontology can be expressed in the language, and which cannot.

Besides the major position in model theory and the key role on understanding expressiveness, model-theoretic characterizations also have many potential implications. For example, model-theoretic characterizations provide a natural way for developing algorithms to identify the language rewritability, that is, to decide whether a given theory or ontology can be rewritten in a simple language. Such algorithms may play important roles in implementing systems for ontology-mediated query answering. Moreover, we are also interested in understanding why the guarded-based restrictions make existential rule languages tractable. We hope our characterizations give an alternative explanation on this question, which may provide a new perspective on exploiting new tractable languages.

Model-theoretic characterizations over finite structures for full TGDs (TGDs without existential quantifier) and equality-generating dependencies had been studied by \citeauthor{MakowskyV1986} in \citeyear{MakowskyV1986}. However, their characterizations involve infinite sets of dependencies. To remedy the finite expressibility in characterizations, some conditions had been proposed, including \citeauthor{Hull1984}'s {\em finite-rank} notion~\shortcite{Hull1984} and \citeauthor{MakowskyV1986}'s {\em locality}~\shortcite{MakowskyV1986}. Yet both of them are not easy to identify. Over finite structures, even for full TGDs, a natural model-theoretic characterization remains open~\cite{tenCateK2014}. For arbitrary structures, except for some simple classes of dependencies such as full TGDs and negative constraints, to the best of our knowledge, no model-theoretic characterization is known for expressive existential rule languages such as TGDs and its guarded-based restrictions.

In this work, we characterize existential rule languages by some natural semantic properties. The addressed languages consist of (disjunctive) embedded dependencies, TGDs, and several important restricted classes of TGDs such as frontier-guarded TGDs, guarded TGDs and linear TGDs, three of the main languages for ontology-mediated query answering~\cite{CaliGLMP2010}. All the semantic properties involved in our characterizations are algebraic relationships among structures, incuding variants of homomorphisms and unions, as well as direct products. Interestingly, except the characterizations w.r.t. first-order logic, all other characterizations hold for both finite structures and arbitrary structures. As a natural application, we also use the obtained characterizations to identify the complexity of rewritability among the above languages.  

\section{Preliminaries}

\subsection{Notations and Conventions} All {\em signatures} involved in this paper are {\em relational}, consisting of a set of {\em constant symbols} and a set of {\em relation symbols}, each of which is armed with a natural number, its {\em arity}. Each {\em term} is either a variable or a constant symbol. 
Given a signature $\tau$, {\em atomic formulas}, {\em (first-order) formulas} and {\em sentences} over $\tau$ are defined as usual. An atomic formula is {\em relational} if it is of the form $R(\vec{t})$ where $R$ is a relation symbol. Given a formula $\varphi$, we write $\varphi(\vec{x})$ if every {\em free variable} of $\varphi$ appears in $\vec{x}$.

Fix $\tau$ as a signature. Every {\em structure $\mathcal{A}$ over $\tau$} (or simply {\em $\tau$-structure}) consists of a nonempty set $A$ called its {\em domain},
a relation $R^{\mathcal{A}}\subseteq A^n$ for each $n$-ary relation symbol $R\in\tau$, and a constant $c^\mathcal{A}\in A$ for each constant symbol $c\in\tau$. A structure is
{\em finite} if its domain is finite, and {\em infinite} otherwise.

Let $\mathcal{A}$ be a $\tau$-structure, and $X$ a subset of $A$ such that $c^\mathcal{A}\in X$ for all constant symbols $c\in\tau$. The {\em substructure} of $\mathcal{A}$ {\em induced} by a set $X\subseteq A$, denoted $\mathcal{A}|_{X}$, is a $\tau$-structure with domain $X$ which interprets each relation symbol $R\in\tau$ as $R^{\mathcal{A}}|_{X}$, and interprets each constant symbol $c\in\tau$ as $c^\mathcal{A}$. A structure $\mathcal{B}$ is called a {\em substructure} of $\mathcal{A}$, or equivalently, $\mathcal{A}$ is called an {\em extension} of $\mathcal{B}$,  if $\mathcal{B}=\mathcal{A}|_X$ for some set $X\subseteq A$. Let $\nu$ be a signature such that $\tau\subseteq\nu$. A $\nu$-structure $\mathcal{B}$ is called a {\em $\nu$-expansion} of $\mathcal{A}$ if they have the same domain and share the same interpretation on every symbol in $\tau$. Suppose $a_1,\dots,a_k\in A$, by $(\mathcal{A},a_1,\dots,a_k)$ we denote the expansion of $\mathcal{A}$ that assigns each constant $a_i$ to a fresh constant symbol.

Let $\mathcal{A}$ and $\mathcal{B}$ be $\tau$-structures. If $\mathcal{A}$ and $\mathcal{B}$ have the same interpretations on constant symbols then let $\mathcal{A}\cup\mathcal{B}$ denote the {\em union} of $\mathcal{A}$ and $\mathcal{B}$, which is a $\tau$-structure with domain $A\cup B$, interpreting $R$ as $R^{\mathcal{A}}\cup R^{\mathcal{B}}$ for each relation symbol $R\in\tau$, and interpreting $c$ as $c^\mathcal{A}$ for each constant symbol $c\in\tau$. We say $\mathcal{A}$ is {\em homomorphic to} $\mathcal{B}$, written $\mathcal{A}\rightarrow\mathcal{B}$, if there is a function $h:A\rightarrow B$ such that (i) $h(c^\mathcal{A})=c^\mathcal{B}$ for all constant symbols $c\in \tau$, and (ii) $h(R^{\mathcal{A}})\subseteq R^{\mathcal{B}}$ for all relation symbols $R\in\tau$. 
We write $\mathcal{A}\rightleftarrows\mathcal{B}$ if both $\mathcal{A}\rightarrow\mathcal{B}$ and $\mathcal{B}\rightarrow\mathcal{A}$ hold.

Let $\mathcal{A}$ be a structure. An {\em assignment} in $\mathcal{A}$ is a function from a set of variables to $A$. Given a tuple $\vec{a}$ of constants in $A$ and a tuple $\vec{x}$ of variables of the same length, we let $\vec{a}/\vec{x}$ denote the assignment that maps the $i$-component of $\vec{x}$ to the $i$-component of $\vec{a}$ for $1\le i\le |\vec{x}|$, where $|\vec{x}|$ denotes the length of $\vec{x}$. Let $s$ be an assignment in $\mathcal{A}$ and $\varphi(\vec{x})$ be a first-order formula. By $\mathcal{A}\models\varphi[s]$ we mean that $\varphi$ is {\em satisfied} by $s$ in $\mathcal{A}$. In particular,  if $\varphi$ is a sentence, we simply written $\mathcal{A}\models\varphi$, and say $\varphi$ is {\em satisfied} in $\mathcal{A}$, or equivalently, $\mathcal{A}$ is a {\em model} of $\varphi$. If the assignment $\vec{a}/\vec{x}$ is clear from the context, we simply use $\varphi[\vec{a}]$ to denote $\varphi[\vec{a}/\vec{x}]$. Let $\Sigma$ be a set of sentences, $\mathcal{A}$ is a {\em model} of $\Sigma$ if $\mathcal{A}\models\varphi$ for all $\varphi\in\Sigma$. Given a sentence $\psi$, we write $\Sigma\vDash\psi$ (resp., $\Sigma\vDash_{\mathrm{fin}}\psi$) if every model (resp., finite model) of $\Sigma$ is also a model of $\psi$.

\subsection{Existential Rule Languages}

A {\em generalized dependency} (GD) is a sentence $\sigma$ of the form 
\begin{equation}
\forall\vec{x}(\phi(\vec{x})\rightarrow\exists\vec{y}(\psi_1(\vec{x},\vec{y})\vee\cdots\vee\psi_n(\vec{x},\vec{y}))
\end{equation}
where $n\ge 0$, and $\phi,\psi_1,\dots,\psi_n$ are conjunctions of atomic formulas. The left-hand (resp., right-hand) side of the implication is called the {\em body} (resp., {\em head}). Variables among $\vec{x}$ and $\vec{y}$ are called {\em universal}, and {\em existential}, respectively. A {\em frontier variables} is a universal variable that occurs in the head. In particular, $\sigma$ is called {\em nondisjunctive} if $n\le 1$, and called a {\em negative constraint} if $n=0$. In the latter case, we write $\sigma$ as 
\begin{equation}
\forall\vec{x}(\phi(\vec{x})\rightarrow\bot).
\end{equation}
For simplicity, we will omit the universal quantifiers and the brackets appearing outside the atoms if no confusion occurs.

Furthermore, a GD $\sigma$ is called {\em safe} if every frontier variable of $\sigma$ has at least one occurrence in some relational atomic formula in the body of $\sigma$. Every {\em disjunctive embedded dependency} (DED) is a safe generalized dependency which is not a negative constraint. 
Every {\em embedded dependency} (ED) is a nondisjunctive DED.
In addition, an ED is called an {\em tuple-generating dependency} (TGD) if it is equality-free.

We will also address several important classes of restricted TGDs. A TGD $\sigma$ is called {\em frontier-guarded} (resp., {\em guarded}) if there is a relational atomic formula $\alpha$ in its body that contains all the frontier (resp., universal) variables of $\sigma$. In either case, $\alpha$ is called the {\em guard} of $\sigma$. Moreover, $\sigma$ is {\em linear} if the body of $\sigma$ consists of exactly one conjunct. Note that all linear TGDs are guarded and all guarded TGDs are frontier-guarded.

\section{Model-theoretic Characterizations}

In this section, we address the model-theoretic characterizations of existential rule languages mentioned above.

\subsection{Generalized Dependencies} We need some notions firstly. Let $\mathcal{A}$ and $\mathcal{B}$ be structures over a signature $\tau$. By a {\em tuple} on $\mathcal{A}$ we mean a finite sequence of constants in $A$. We say that $\mathcal{A}$ is globally-homomorphic to $\mathcal{B}$, written $\mathcal{A}\Rightarrow\mathcal{B}$, if there is a function $\pi$ that maps each tuple $\vec{a}$ on $\mathcal{A}$ to a tuple $\pi(\vec{a})$ on $\mathcal{B}$ such that $(\mathcal{A},\vec{a})\rightleftarrows(\mathcal{B},\pi(\vec{a}))$; in this case, we call $\pi$ a {\em global homomorphism} from $\mathcal{A}$ to $\mathcal{B}$, and call $\mathcal{A}$ a {\em globally-homomorphic preimage} of $\mathcal{B}$. 

Given a first-order sentence $\varphi$ over $\tau$, we say that $\varphi$ is {\em preserved under globally-homomorphic preimages [in the finite]} if for all [finite] $\tau$-structures $\mathcal{A}$ and $\mathcal{B}$, if $\mathcal{A}$ is globally homomorphic to $\mathcal{B}$ and $\mathcal{B}$ is a model $\varphi$, then $\mathcal{A}$ is also a model of $\varphi$. Notice that every sentence preserved under globally-homomorphic preimages is also preserved under globally-homomorphic preimages in the finite, but not vice versa. 

By a routine check, it is easy to prove the following:

\begin{proposition}\label{prop:gd_ghom_prsv}
Every GD is preserved under globally-homomorphic preimages [in the finite].
\end{proposition}


To establish the desired characterization, we hope that the preservation under globally-homomorphic preimages is not too powerful. The following is a very simple example which is slightly beyond the class of GDs but already not preserved under globally-homomorphic preimages in the finite.

\begin{example}\label{exm:ghom_prv_cntexp}
Let $\psi$ denote $\exists x\neg{Q}(x)$ and $\tau=\{Q\}$. Let $\mathcal{A}$ be a $\tau$-structure with $A=\{a,b\}$ and $Q^\mathcal{A}=\{a\}$. Let $\mathcal{B}$ be the substructure of $\mathcal{A}$ induced by $\{a\}$. Clearly, $\mathcal{B}$ is globally homomorphic to $\mathcal{A}$. It is also easy to see that $\mathcal{A}$ is a model of $\psi$, but $\mathcal{B}$ is not, which implies that $\psi$ is not preserved under globally-homomorphic preimages even in the finite.
\end{example}

The following theorem establishes the desired characterizations for the class of finite sets of GDs.

\begin{theorem}\label{thm:ghom_prsv}
A first-order sentence is equivalent to a finite set of GDs iff it is preserved under globally-homomorphic preimages.
\end{theorem}

To prove this theorem, we need some notions and lemmas. Let $\mathcal{A}$ and $\mathcal{B}$ be structures over a signature $\tau$. Given a class $\mathcal{C}$ of sentences over $\tau$, we write $\mathcal{A}\preceq_{\mathcal{C}}\mathcal{B}$ if for all sentences $\varphi\in\mathcal{C}$, $\mathcal{A}\models\varphi$ implies $\mathcal{B}\models\varphi$. For simplicity, we simply drop the subscript $\mathcal{C}$ if $\mathcal{C}$ is the class of all first-order sentences over $\tau$. We write $\mathcal{A}\equiv\mathcal{B}$ if both $\mathcal{A}\preceq\mathcal{B}$ and $\mathcal{B}\preceq\mathcal{A}$ hold. 

We write $\Gamma(x)$ to denote that every formula in $\Gamma(x)$ has exactly one free variable $x$. We say $\Gamma(x)$ is {\em realized} in a structure $\mathcal{A}$ if there is some $a\in A$ such that $\mathcal{A}\models\vartheta[a/x]$ for all formulas $\vartheta(x)\in\Gamma(x)$. By $Th(\mathcal{A})$ we denote the class of all first-order sentences satisfied in $\mathcal{A}$. We say $\mathcal{A}$ is {\em $\omega$-saturated} if for every finite set $X\subseteq A$, every set $\Gamma(x)$ of formulas consistent with $Th((\mathcal{A},a)_{a\in X})$ is realized in $(\mathcal{A}, a)_{a\in X}$. It is well-known~\cite{CK1990} that for every structure $\mathcal{A}$ there is a $\omega$-saturated structure $\mathcal{B}$ such that $\mathcal{A}\equiv\mathcal{B}$.

Every {\em existential-positive} formula is a first-order formula built on atomic formulas and negated atomic formulas by using connectives $\wedge,\vee$ and the quantifier $\exists$. Let $\exists^+$ denote the class of existential-positive sentence. It is easy to prove:

\begin{lemma}\label{hmclass}
Let $\mathcal{A}$ and $\mathcal{B}$ be structures over the same signature. Then both of the following are true:
 \begin{enumerate}
  \item If $\mathcal{A} \rightarrow \mathcal{B}$ then $\mathcal{A}
    \preceq_{\exists^+} \mathcal{B}$.
  \item If $\mathcal{A} \preceq_{\exists^+} \mathcal{B}$ and $\mathcal{B}$ is
    $\omega$-saturated then $\mathcal{A} \rightarrow \mathcal{B}$.
  \end{enumerate}
\end{lemma}

Let $\mathsf{GD}$ denote the class of finite sets fo generalized dependencies. With Lemma~\ref{hmclass}, we are able to prove the following:

\begin{lemma}
\label{globallem}
  Let $\mathcal{A}$ and $\mathcal{B}$ be $\omega$-saturated structures over the same signature. If $\mathcal{B}\preceq_{\mathsf{GD}}\mathcal{A}$ then $\mathcal{A}\Rightarrow\mathcal{B}$.
\end{lemma}

\begin{proof}
 Assume $\mathcal{B}\preceq_{\mathsf{GD}}\mathcal{A}$. We need to prove $\mathcal{A}\Rightarrow\mathcal{B}$. By Lemma~\ref{hmclass}, it suffices to show that for each tuple $\vec{a}$ on $\mathcal{A}$ there is a tuple $\pi(\vec{a})$ such that $(\mathcal{B},\pi(\vec{a}))\preceq_{\mathsf{GD}}(\mathcal{A},\vec{a})$. Note that, by Proposition 5.1.1 in~\cite{CK1990}, $(\mathcal{B},\pi(\vec{a}))$ and $(\mathcal{A},\vec{a})$ are $\omega$-saturated; so Lemma~\ref{hmclass} is applicable.

The desired statement can be done by an induction on the length of $\vec{a}$. It is trivial for the case where $|\vec{a}|=0$. Assume as induction hypothesis that the desired statement holds for $|\vec{a}|=k\ge 0$, we need to prove that it also holds for the case where $\vec{a}=k+1$. Suppose $\vec{a}=(\vec{a}_0,a)$. By inductive hypothesis, there is a tuple $\vec{b}_0$ such that \vspace{-.1cm}
 \begin{equation}\label{eqn:inductive_assuption}
 (\mathcal{B},\vec{b}_0)\preceq_{\mathsf{GD}}(\mathcal{A},\vec{a}_0). \vspace{-.1cm}
 \end{equation}
 Let $\Gamma(x)$ be the class of existential-positive formulas and their negations such that $(\mathcal{A},\vec{a}_0)\models \varphi[a/x]$ for all $\varphi(x)\in\Gamma(x)$. To prove the existence of a constant $b\in B$ such that \vspace{-.1cm}
\begin{equation}
(\mathcal{B},\vec{b}_0,b)\preceq_{\mathsf{GD}}(\mathcal{A},\vec{a}_0,a),\vspace{-.1cm}
\end{equation}
by the $\omega$-saturatedness of $\mathcal{B}$, it suffices to show that every finite subset of $\Gamma(x)$ is realized in $(\mathcal{B},\vec{b}_0)$. Let $\Gamma_0(x)$ be any finite subset of $\Gamma(x)$. Let $\varphi(x)$ denote the conjunction of all formulas in $\Gamma_0(x)$, and let $\psi=\neg\exists x\varphi(x)$. Clearly, $\psi$ is equivalent to a finite set of GDs and $(\mathcal{A},\vec{a}_0)\not\models\psi$. By the inductive assumption~(\ref{eqn:inductive_assuption}), we know $(\mathcal{B},\vec{b}_0)\not\models\psi$, or equivalently, there exists a constant $b'\in B$ such that $(\mathcal{B},\vec{b}_0)\models\varphi[b'/x]$. Consequently, $\Gamma_0(x)$ is realized in $(\mathcal{B},\vec{b}_0)$, which is as desired.
\end{proof}

Now we are able to prove the desired theorem.

\begin{proof}[Proof of Theorem~\ref{thm:ghom_prsv}]
(Only-if) By Proposition~\ref{prop:gd_ghom_prsv}.

\smallskip
(If) We assume that $\varphi$ is a first-order sentence preserved under globally-homomorphic preimages. Let $\mathrm{con}(\varphi)$ denote the class of all GDs that are logical consequences of $\varphi$. We want to show that $\mathrm{con}(\varphi)$ is equivalent to $\varphi$, which implies the desire result by the compactness. Let $\mathcal{A}$ be any model of $\mathrm{con}(\varphi)$. It suffices to show that $\mathcal{A}$ is also a model of $\varphi$.
Let\vspace{-.05cm}
\begin{align*}
\Sigma&=\{\neg\gamma:\gamma\in\mathsf{GD}\,\&\,\,\mathcal{A}\models\neg\gamma\}.\vspace{-.05cm}
\end{align*}
Now we prove the following property:
\smallskip

\noindent{\em Claim.} $\Sigma\cup\{\varphi\}$ is satisfiable.
\smallskip

Let $\Sigma_0$ be an arbitrary finite subset of $\Sigma$. To show the claim, by the compactness, it suffices to show that $\Sigma_0\cup\{\varphi\}$ is satisfiable. Towards a contradiction, assume that this is not the case. Suppose $\Sigma_0=\{\neg\gamma_1,\dots,\neg\gamma_n\}$, and let $\psi$ denote the formula $\gamma_1\vee\cdots\vee\gamma_n$. Then we must have $\varphi\vDash\psi$. It is not difficult to see that $\psi$ is equivalent to a GD (by renaming the individual variables and lifting the universal quantifiers, and then by a routine transformation). Thus, $\mathcal{A}$ should be a model of $\psi$. This implies that there is some integer $i:1\le i\le n$ such that $\mathcal{A}\models\gamma_i$, which contradicts with $\gamma_i\in\Sigma$ and the definition of $\Sigma$. So, we obtain the claim.

Applying the above claim, there is thus a model, say $\mathcal{B}$, of $\Sigma\cup\{\varphi\}$. Consequently, we have $\mathcal{B}\preceq_{\mathsf{GD}}\mathcal{A}$. Let $\mathcal{A}^+$ and $\mathcal{B}^+$ be $\omega$-saturated structures such that $\mathcal{A}\equiv\mathcal{A}^+$ and $\mathcal{B}\equiv\mathcal{B}^+$. Then $\mathcal{B}^+\preceq_{\mathsf{GD}}\mathcal{A}^+$ is clearly true, and $\mathcal{B}^+$ is a model of $\varphi$. By Lemma \ref{globallem}, $\mathcal{A}^+$ is then globally homomorphic to $\mathcal{B}^+$. Since $\varphi$ is preserved under globally-homomorphic preimages, $\mathcal{A}^+$ should be a model of $\varphi$, which implies that $\mathcal{A}$ is also a model of $\varphi$. This completes the proof of the theorem.
\end{proof}

Note that the above argument only works on the class of arbitrary structures. Over finite structures, the  characterization is in general not true, as shown by the following proposition.

\begin{theorem}\label{cntexm_retrprsv}
There is a first-order sentence that is preserved under globally-homomorphic preimages in the finite but is not equivalent to any finite set of GDs over finite structures.
\end{theorem}

This can prove by constructing an example, which can be done by a slight modification to Gurevich and Shelah's counterexample (see, e.g., Theorem 2.1.1 in~\cite{Rosen2002}).

\subsection{Disjunctive Embedded Dependencies}

According to the definition, DEDs are safe GDs that are not negative constraints. So,  
to characterize DEDs, we need some properties to assure the safeness and to avoid occurrences of negative constraints. To do the latter, we use a technique called {\em trivial structure}, which was used in~\cite{MakowskyV1986} to characterize full TGDs. 

We first recall the notion of trivial structure. A structure $\mathcal{A}$ is called {\em trivial} if the domain of $\mathcal{A}$ consists of exactly one element and every relation symbol in the signature is interpreted by $\mathcal{A}$ as a full relation on the domain of a proper arity.

To capture the safeness of a DED, we propose a similar notion. A structure $\mathcal{A}$ is called {\em sharp} if all the following hold:
\begin{itemize}
\item the domain of $\mathcal{A}$ consists of exactly two distinct constants, say $\ast$ and $\circ$;
\item for each constant symbol $c$ in the signature, $c^\mathcal{A}=\ast$;
\item for each relation symbol $R$ in the signature, $R^\mathcal{A}$ consists of exactly a single tuple $(\ast,\dots,\ast)$ of a proper length. 
\end{itemize}

The following example shows that the sharp models are able to separate the class of DEDs from the class of GDs:

\begin{example}
Let $\sigma$ be a DED of the following form:
\begin{equation}
P(x)\wedge R(x,y)\rightarrow Q(y).
\end{equation}
Let $\tau=\{P,Q,R\}$, and let $\mathcal{A}$ be a $\tau$-structure with the domain $\{a,b\}$, interpreting both $P$ and $Q$ as $\{a\}$, and interpreting $R$ as $\{(a,a)\}$. Clearly, $\mathcal{A}$ is a sharp model of $\sigma$. 

Let $\sigma_0$ denote the GD obtained from $\sigma$ by replacing $R(x,y)$ with $R(x,x)$. Clearly, $\sigma_0$ is a GD that is not satisfied in $\mathcal{A}$. 
\end{example}

\smallskip
The following result can be shown by a routine check:
 
\begin{proposition}\label{prop:trivial_sharp}
Let $\Sigma$ be a finite set of GDs. Then all the following properties are equivalent:
\begin{enumerate}
\item $\Sigma$ is equivalent to a finite set of DEDs;
\item $\Sigma$ is equivalent to a finite set of DEDs over finite structures;
\item $\Sigma$ has both a trivial model and a sharp model.
\end{enumerate} 
\end{proposition}

Note that both ``$\varphi$ has a trivial model" and ``$\varphi$ has a sharp model" can be regarded as trivial preservation properties.

\subsection{Embedded Dependencies}
To characterize EDs, we use the notion of direct products. 
Let $\mathcal{A}$ and $\mathcal{B}$ be structures over a signature $\tau$. The {\em direct product} of $\mathcal{A}$ and $\mathcal{B}$, denoted $\mathcal{A}\times\mathcal{B}$, is a $\tau$-structure defined as follows:
\begin{itemize} 
\item the domain of $\mathcal{A}\times\mathcal{B}$ is $A\times B$;
\item for all constant symbols $c\in\tau$, $c^{\mathcal{A}\times\mathcal{B}}=\langle c^\mathcal{A},c^\mathcal{B}\rangle$;
\item for all $k$-ary relation symbols $R\in\tau$, all tuples $\vec{a}$ on $\mathcal{A}$, and all tuples $\vec{b}$ on $\mathcal{B}$, 
$(\langle a_1,b_1\rangle,\dots,\langle a_k,b_k\rangle)\in{R}^{\mathcal{A}\times\mathcal{B}}$ if $\vec{a}\in{R}^{\mathcal{A}}$ and $\vec{b}\in{R}^{\mathcal{B}}$, where $a_i$ and $b_i$ denote the $i$-th component of $\vec{a}$ and $\vec{b}$, respectively.
\end{itemize}

We say a sentence  $\varphi$ is {\em preserved under direct products [in the finite]} if, for any [finite] models $\mathcal{A}$ and $\mathcal{B}$ of $\varphi$, $\mathcal{A}\times\mathcal{B}$ is also a model of $\varphi$. 

\smallskip
The following can be shown by a routine check. 

\begin{proposition}\label{prop:d_dp_prsv}
Every ED is preserved under direct products [in the finite].
\end{proposition}

In general, the direct product preservation fails for DEDs. A simple counterexample is given as follows: 

\begin{example}\label{exm:dproduct_prv_cntexp}
Let $\sigma$ denote the DED $R\rightarrow S\vee T$ where $R,S$ and $T$ are nullary relation symbols. Let $\tau$ be the signature $\{R,S,T\}$. Let $\mathcal{A}$ and $\mathcal{B}$ be $\tau$-structures such that
\begin{itemize}
\item $\mathcal{A}$ and $\mathcal{B}$ have the same domain $\{a\}$; 
\item $R^{\mathcal{A}}=R^{\mathcal{B}}=S^{\mathcal{A}}=T^{\mathcal{B}}=true$, $S^{\mathcal{B}}=T^{\mathcal{A}}=false$.
\end{itemize} 
Clearly, both $\mathcal{A}$ and $\mathcal{B}$ are models of $\sigma$, but $\mathcal{A}\times\mathcal{B}$ is not. Thus, $\sigma$ is not preserved under direct products even in the finite.
\end{example}

The following result shows that the property of direct product preservation exactly captures the class of DEDs in which the disjunctions can be eliminated. Interestingly, this characterization also works over finite structures.

\begin{theorem}\label{prop:dgd_dp_prsv}
A finite set of DEDs is equivalent to a finite set of EDs [over finite structures] iff 
it is preserved under direct products [in the finite]. 
\end{theorem}

\subsection{Tuple-generating Dependencies}

Let $\mathcal{A}$ and $\mathcal{B}$ be structures over a signature $\tau$. A {\em strict homomorphism from $\mathcal{A}$ into (resp., onto) $\mathcal{B}$} is a function $h$ from $A$ into (resp., onto) $B$ such that
\begin{itemize}
\item for every relation symbol ${R}\in\tau$ and every tuple $\vec{a}$ on $\mathcal{A}$ of a proper length, we have $\vec{a}\in{R}^\mathcal{A}$ iff $h(\vec{a})\in{R}^\mathcal{B}$, and
\item for every constant symbol ${c}\in\tau$, we have $h({c}^{\mathcal{A}})={c}^{\mathcal{B}}$.
\end{itemize}
If such a strict homomorphism exists, we say $\mathcal{B}$ is a {\em strictly-homomorphic image} of $\mathcal{A}$, and say $\mathcal{A}$ is, conversely, a {\em strictly-homomorphic preimage} of $\mathcal{B}$. A sentence $\varphi$ is said to be {\em preserved under strictly-homomorphic (pre)images [in the finite]} if, for every [finite] model $\mathcal{A}$ of $\varphi$ and every [finite] strictly-homomorphic (pre)image $\mathcal{B}$ of $\mathcal{A}$, $\mathcal{B}$ is also a model of $\varphi$. 

\smallskip
The following gives us the desired characterazations:

\begin{theorem}\label{thm:strict_hom_prsv}
A finite set of EDs is equivalent to a finite set of TGDs [over finite structures] iff
it is preserved under both strictly-homomorphic images and preimages [in the finite].
\end{theorem}

Interestingly, it is not difficult to show that, if no constant symbol is involved, the strictly-homomorphic image preservation can be removed from the characterization. 

\subsection{Frontier-guarded TGDs}\label{sec:fgd} 

To characterize frontier-guarded TGDs, we first define some notations. 
Let $\mathcal{A}$ be a structure. We define $\{\mathcal{A}_X:X\subseteq A\}$ as a family of structures over the same signature such that
\begin{itemize}
\item for all $X\subseteq A$, there is an isomorphism $p_X$ from $\mathcal{A}$ to $\mathcal{A}_X$ such that $p_X(a)=a$ for all $a\in X$; 
\item for all $X,Y\subseteq A$, $A_X\cap A_Y= X\cap Y$, where $A_X$ and $A_Y$ denote the domains of $\mathcal{A}_X$ and $\mathcal{A}_Y$, respectively.
\end{itemize} 
Every {\em guarded set} of $\mathcal{A}$ is defined as a finite subset $X$ of $A$ that contains all interpretations of constant symbols in $\mathcal{A}$.
A sentence $\varphi$ is said to be {\em preserved under isomorphic unions [in the finite]} if, for all [finite] models $\mathcal{A}$ of $\varphi$ and all finite sets $\mathbb{G}$ of guarded sets of $A$, $\bigcup_{X\in \mathbb{G}}\mathcal{A}_{X}$ is also a model of $\varphi$.

\begin{example}\label{exm:iso_union}
Let $\tau$ denote $\{R\}$ where $R$ is a binary relation symbol. Let $\mathcal{A}$ be a $\tau$-structure defined as follows:
\begin{itemize}
\item the domain $A$ consists of two distinct constants $a$ and $b$;
\item the relation symbol ${R}$ is interpreted as $A\times A$. 
\end{itemize}
Let $X=\{a\}$, $Y=\{a,b\}$, and $\mathbb{G}=\{X,Y\}$. Then $\mathcal{A}_X,\mathcal{A}_Y$ and $\bigcup_{Z\in\mathbb{G}}\mathcal{A}_Z$ are $\tau$-structures illustrated by Figure 1. 
\begin{figure}
\begin{center}
\begin{tabular}{cccc}
\includegraphics[width=0.17\columnwidth,height=0.09\paperheight]{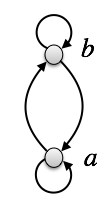} &
\includegraphics[width=0.17\columnwidth,height=0.09\paperheight]{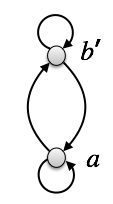} &
\includegraphics[width=0.17\columnwidth,height=0.09\paperheight]{fig_A} &
\includegraphics[width=0.30\columnwidth,height=0.09\paperheight]{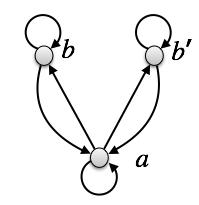}
\\
$\mathcal{A}$&
$\mathcal{A}_X$&
$\mathcal{A}_Y$&
$\mathcal{A}_X\cup \mathcal{A}_Y$
\end{tabular}\\
\end{center}\vspace{-.25cm}
\caption{ Isomorphic Union in Example~\ref{exm:iso_union}}\vspace{-.25cm}
\end{figure}%
\end{example}

By a routine check, one can prove the following property: 

\begin{proposition}\label{prop:tuple_iso_copy_prsv}
Every frontier-guarded TGD is preserved under isomorphic unions [in the finite].
\end{proposition}


Now, a natural question arises as to whether the isomorphic union preservation is able to separate frontier-guarded TGDs from TGDs. The following example shows that it is true. 

\begin{example}[Example~\ref{exm:iso_union} cont.]
Let $\sigma$ denote the TGD
\begin{equation}
{R}(x,y)\wedge{R}(y,z)\rightarrow{R}(x,z)
\end{equation}
and let $\mathcal{A}$ be the structure defined in Example~\ref{exm:iso_union}.
Then it is easy to see that $\mathcal{A}$ is a model of $\sigma$ but $\bigcup_{Z\in\mathbb{G}}\mathcal{A}_{Z}$ is not. So, $\sigma$ is not preserved under isomorphic unions even in the finite.
\end{example}

The following result provides the desired characterization. Note that the characterization also holds over finite structures.

\begin{theorem}\label{thm:char_d2fgd}
A finite set of TGDs is equivalent to a finite set of frontier-guarded TGDs [over finite structures] iff it is preserved under isomorphic unions [in the finite]. 
\end{theorem}

Every {\em conjunctive query} (CQ) is a first-order formula of the form $\exists\vec{y}\vartheta(\vec{x},\vec{y})$ where $\vartheta$ is a conjunction of relational atomic formulas. Now we first present a lemma as follows:

\begin{lemma}\label{lem:cq_copyprsv}
Let $\phi(\vec{x})$ be a CQ, $\tau$ the signature $\tau$ of $\phi$, $\mathcal{A}$ a $\tau$-structure, $\vec{a}$ a tuple on $\mathcal{A}$ with $|\vec{a}|=|\vec{x}|$, and $\mathbb{G}$ a finite set of guarded sets of $\mathcal{A}$ such that every constant in $\vec{a}$ belongs to some $X\in\mathbb{G}$. If $\bigcup_{X\in\mathbb{G}}\mathcal{A}_{X}\models\phi[\vec{a}]$ then $\mathcal{A}\models\phi[\vec{a}]$. 
\end{lemma}

Now we are in the position to prove the theorem.\vspace{-.1cm}

\begin{proof}[Sketched Proof of Theorem~\ref{thm:char_d2fgd}] (Only-if) By Proposition~\ref{prop:tuple_iso_copy_prsv}. 

\smallskip
(If) Only address arbitrary structures. A slight modification to the following argument applies to finite structures.

Let $\Sigma$ be a finite set of TGDs preserved under isomorphic unions. We first show that $\Sigma$ is equivalent to a set of {\em diverse dependencies}, each of which is a sentence of the form
\begin{equation}\label{eqn:diverse_dependency}
\forall\vec{x}(\lambda_{\mathrm{una}}(\vec{x})\wedge\phi(\vec{x})\rightarrow\exists\vec{y}\psi(\vec{x},\vec{y}))
\end{equation} 
where 
$\phi$ and $\psi$ are conjunctions of relational atomic formulas, and $\lambda_{\mathrm{una}}(\vec{x})$ denotes $\bigwedge_{1\le i<j\le k}\neg t_i=t_j$ if $t_1,\dots,t_k$ be an enumeration (without repetition) of all constant symbols and universal variables in $\phi$ and $\psi$. It is easy to show

\smallskip
\noindent{\em Claim 1.} $\Sigma$ is equivalent to a finite set of diverse dependencies. 
\smallskip

To present the proof, more notions are needed. 
Let $\sigma$  be a diverse dependency of the form~(\ref{eqn:diverse_dependency}). 
The {\em graph} of $\sigma$ is defined as an undirected graph with each conjunct of $\psi$ as a vertex and with each pair of conjuncts of $\psi$ that share some existential variable as an edge. We say that $\sigma$ is {\em quasi-frontier-guarded} if, for every connected component $\delta$ of the graph of $\sigma$, the set of variables that occurs in both $\delta$ and $\vec{x}$ (the tuple of universal variables of $\sigma$) co-occur in some atomic formula in $\phi$.

Let $\Gamma$ be a finite set of diverse dependencies that is equivalent to $\Sigma$. 
Take $\gamma\in\Gamma$ as a diverse dependency of the form~(\ref{eqn:diverse_dependency}).  
Let $S_\gamma$ denote the set of substitutions, which only map existential variables to some terms in $\gamma$, such that $s({\gamma})$ is a quasi-frontier-guarded diverse dependency. Let $\gamma^\ast$ denote\vspace{-.1cm}
\begin{equation}\label{eqn:specialization}
\forall\vec{x}\left[\lambda_{{\mathrm{una}}}(\vec{x})\wedge\phi(\vec{x})\rightarrow\exists\vec{y}\bigvee_{s\in S_\gamma}s(\psi)(\vec{x},\vec{y})\right]\vspace{-.1cm}
\end{equation}
and let $\Gamma^\ast$ be the set of $\gamma^\ast$ for all $\gamma\in\Gamma$. We want to prove that $\Gamma^\ast$ is equivalent to $\Sigma$. The direction $\Gamma^\ast\vDash\Sigma$ follows from the definition of $\Gamma^\ast$. To show the converse, it suffices to prove

\smallskip
\noindent{\em Claim 2.} $\Sigma\vDash\gamma^\ast$ for all $\gamma\in\Gamma$. \vspace{-.1cm}

\begin{proof}
Let $\mathcal{A}$ be a model of $\Sigma$ and $\vec{a}$ a tuple on $\mathcal{A}$ such that $\mathcal{A}\models\lambda_{\mathrm{una}}[\vec{a}]$ and $\mathcal{A}\models\phi[\vec{a}]$. Let $C$ be the set of all interpretations of constant symbols in $\mathcal{A}$. Let $\mathbb{G}$ be the set of guarded sets of $\mathcal{A}$ such that if $X\in\mathbb{G}$ then all constants in $X\setminus C$ co-occur in an atomic formula in $\phi(\vec{a})$. Let $\mathcal{B}=\bigcup_{X\in\mathbb{G}}\mathcal{A}_{X}$. By definition we know $\mathcal{B}\models\lambda_{\mathrm{una}}[\vec{a}]$ and $\mathcal{B}\models\phi[\vec{a}]$. As $\Sigma$ is preserved under isomorphic unions, $\mathcal{B}$ must be a model of $\Sigma$. Consequently, $\mathcal{B}$ is a model of $\gamma$. We thus have that $\mathcal{B}\models\exists\vec{y}\psi[\vec{a}/\vec{x}]$, i.e., there is a tuple $\vec{b}$ on $\mathcal{B}$ with $\mathcal{B}\models\psi[\vec{a},\vec{b}]$. 

Define a substitution $s$ as follows: Given $i=1,\dots,|\vec{y}|$, let $s(y_i)=c$ if for some constant symbol $c$ with $b_i=c^\mathcal{A}$; if no such $c$ then let $s(y_i)=x_j$ for some $j$ with $b_i=a_j$; if no such $j$ either then let $s(y_i)=y_i$, where $a_i,b_i,x_i,y_i$ denote the $i$-th components of $\vec{a},\vec{b},\vec{x},\vec{y}$, respectively. Clearly $\mathcal{B}\models s(\exists\vec{y}\psi)[\vec{a}/\vec{x}]$. By Lemma~\ref{lem:cq_copyprsv}, we have $\mathcal{A}\models s(\exists\vec{y}\psi)[\vec{a}/\vec{x}]$. By a careful check, one can show $s\in S_\gamma$, i.e., $s(\gamma)$ is quasi-frontier-guarded as desired. We omit the proof here. 
\end{proof} \vspace{-.1cm}

With Claim 2, we then have that $\Gamma^\ast$ is equivalent to $\Sigma$. Take $\gamma\in\Gamma$ and suppose $\gamma^\ast$ is of the form~(\ref{eqn:specialization}). It is easy to see that $\gamma^\ast$ can be equivalently rewritten as a sentence $\gamma^\dag$ of the form\vspace{-.1cm}
\begin{equation}
\forall\vec{x}\left[\phi(\vec{x})\rightarrow\bigvee_{s\in S_\gamma}\exists\vec{y}s(\psi)\vee\bigvee_{1\le i<j\le k}x_i=x_j\right]\vspace{-.1cm}
\end{equation}
where $t_1,\dots,t_k$ is an enumeration (without repetition) of all terms in $\gamma$. Let $\Gamma^\dag$ consist of $\gamma^\dag$ for all $\gamma\in\Gamma$, and let $\Delta(\gamma)$ be a set that consists the TGD
\begin{equation}\label{eqn:reduced_TGD}
\forall\vec{x}(\phi(\vec{x})\rightarrow\exists\vec{y}s(\psi))
\end{equation} for all $s\in S_\gamma$, and $\Delta$ the union of $\Delta(\gamma)$ for all $\gamma\in\Gamma$. Let $\mathrm{con}(\Gamma^\dag)$ denote the set of TGDs $\sigma\in\Delta$ such that $\Gamma^\dag\vDash\sigma$. It is easy to see that each TGD in $\mathrm{con}(\Gamma^\dag)$ is equivalent to a finite number of frontier-guarded TGDs. To complete the proof, it is thus sufficient to show the following property:

\smallskip
\noindent{\em Claim 3.} $\mathrm{con}(\Gamma^\dag)$ is equivalent to $\Gamma^\dag$. 
\smallskip

This can be proved by combining the direct-product argument that proves Theorem~\ref{prop:dgd_dp_prsv} with the strictly-homomorphic preimage preservation argument that proves Theorem~\ref{thm:strict_hom_prsv}.
\end{proof}

\subsection{Guarded TGDs}\label{sec:gd}

We say that a sentence $\varphi$ is {\em preserved under disjoint unions [in the finite]} if, for each pair of [finite] models $\mathcal{A}$ and $\mathcal{B}$ of $\varphi$, $\mathcal{A}\cup\mathcal{B}$ is also a model of $\varphi$ if both the following hold: (i) $\mathcal{A}$ and $\mathcal{B}$ have the same interpretations on constant symbols, and
(ii) if $X=A\cap B$ and $X\ne\emptyset$ then $\mathcal{A}|_{X}=\mathcal{B}|_{X}$.

\begin{proposition}\label{prop:gd_duprv}
Every guarded TGD is preserved under disjoint unions [in the finite].
\end{proposition}

The following example shows that the above property separates guarded TGDs from frontier-guarded TGDs.

\begin{example}\label{exm:gdu_prsv_cntexm}
Let $\sigma$ be the following frontier-guarded TGD:\vspace{-.1cm}
\begin{equation}
{E}(x,y)\wedge{E}(y,z)\rightarrow{C}(y)\vspace{-.1cm}
\end{equation}
and let $\tau=\{{C},{E}\}$. Let $\mathcal{A}$ and $\mathcal{B}$ be $\tau$-structures defined by:\vspace{-.05cm}
\begin{itemize}
\item the domain of $\mathcal{A}$ is $\{a,b\}$ and the domain of $\mathcal{B}$ is $\{b,c\}$;\vspace{-.05cm}
\item ${C}^{\mathcal{A}}={C}^{\mathcal{B}}=\emptyset$, ${E}^{\mathcal{A}}=\{(a,b)\}$, and ${E}^{\mathcal{B}}=\{(b,c)\}$.\vspace{-.05cm}
\end{itemize}
Let $X=A\cap B=\{b\}$. Clearly, $\mathcal{A}|_{X}=\mathcal{B}|_{X}$. By definition, $\mathcal{A}\cup\mathcal{B}$ is a $\tau$-structure with $\{a,b,c\}$ as domain, interpreting $C$ as $\emptyset$, and interpreting $E$ as $\{(a,b),(b,c)\}$.
It is easy to see that both $\mathcal{A}$ and $\mathcal{B}$ are models of $\sigma$, but $\mathcal{A}\cup\mathcal{B}$ is not. So, $\sigma$ is not preserved under disjoint unions even in the finite. 
\end{example}

Now, let us present the desired characterization.

\begin{theorem}\label{thm:char_tgd2gd}
A finite set of TGDs is equivalent to a finite set of guarded TGDs [over finite structures]
iff it is preserved under disjoint unions [in the finite]. 
\end{theorem}

The general idea of proving the hard direction is as follows: First show that every finite set of frontier-guarded TGDs preserved under disjoint unions [in the finite] is equivalent to a finite set of guarded TGDs [over finite structures]. As the disjoint union preservation always implies the isomorphic union preservation, by  Theorem~\ref{thm:char_d2fgd}, we then have the desired result.

\subsection{Linear TGDs}

Every sentence $\varphi$ is said to be {\em preserved under unions [in the finite]} if, for all [finite] models $\mathcal{A}$ and $\mathcal{B}$ of $\varphi$ with the same interpretations on constant symbols,  $\mathcal{A}\cup\mathcal{B}$ is a model of $\varphi$. 

\smallskip
The following theorem was obtained by~\citeauthor{CateFK2015}:

\begin{theorem}[\cite{CateFK2015}]\label{thm:char_tgd2ld_finite}
A finite set of TGDs is equivalent to a finite set of linear TGDs over finite structures iff it is preserved under unions in the finite.
\end{theorem}

To separate the class of linear TGDs from guarded TGDs, a simple example is presented as follows: 

\begin{example}
Let $\sigma$ denote the following guarded TGD:\vspace{-.1cm}
\begin{equation}
{P}(x)\wedge{Q}(x)\rightarrow{R}(x).\vspace{-.1cm}
\end{equation}
Let $\tau$ denote $\{{P},{Q},{R}\}$. Let $\mathcal{A}$ and $\mathcal{B}$ be $\tau$-structures with the same domain $\{a\}$ such that ${P}^{\mathcal{A}}={Q}^{\mathcal{B}}={R}^{\mathcal{A}}={R}^{\mathcal{B}}=\emptyset$ and ${P}^{\mathcal{B}}={Q}^{\mathcal{A}}=\{a\}$. Then it is obvious that both $\mathcal{A}$ and $\mathcal{B}$ are models of $\sigma$. However, $\mathcal{A}\cup\mathcal{B}$ does not satisfy $\sigma$. Therefore, $\sigma$ is not preserved under unions even in the finite.
\end{example}

It is worth noting that \citeauthor{CateFK2015}'s proof of Theorem~\ref{thm:char_tgd2ld_finite} does not work over arbitrary structures. Fortunately, thanks to Theorem~\ref{thm:char_d2fgd} and the finite model property of frontier-guarded TGDs, we are able to show the following characterization:

\begin{theorem}\label{thm:char_tgd2ld}
A finite set of TGDs is equivalent to a finite set of linear TGDs iff it is preserved under unions.
\end{theorem}

\section{Application: Complexity of Rewritability}\label{sec:app}

As a direct application, we use the obtained model-theoretic characterizations to identify complexity bounds of language rewritability. Let $\textsc{PTime}$ (resp., \textsc{PSpace}, \textsc{2ExpTime}) denote the class of languages accepted by some deterministic Turing machine in a polynomial time (resp., polynomial space, double-exponential time). By (\textsc{co})\textsc{RE} we mean (the complement of) the class of recursively enumerable languages.

Let $\mathsf{FO}$ denote the class of all first-order sentences. Let $\mathsf{GD}$ (resp., $\mathsf{DED}$, $\mathsf{ED}$, $\mathsf{TGD}$, $\mathsf{FGTGD}$, $\mathsf{GTGD}$ and $\mathsf{LTGD}$) denote the class of all finite sets of GDs (resp., DEDs, EDs, TGDs, frontier-guarded TGDs, guarded TGDs and linear TGDs). 

\begin{figure}
\begin{center}
\includegraphics[width=1.01\columnwidth,viewport=1 0 795 190,clip=true]{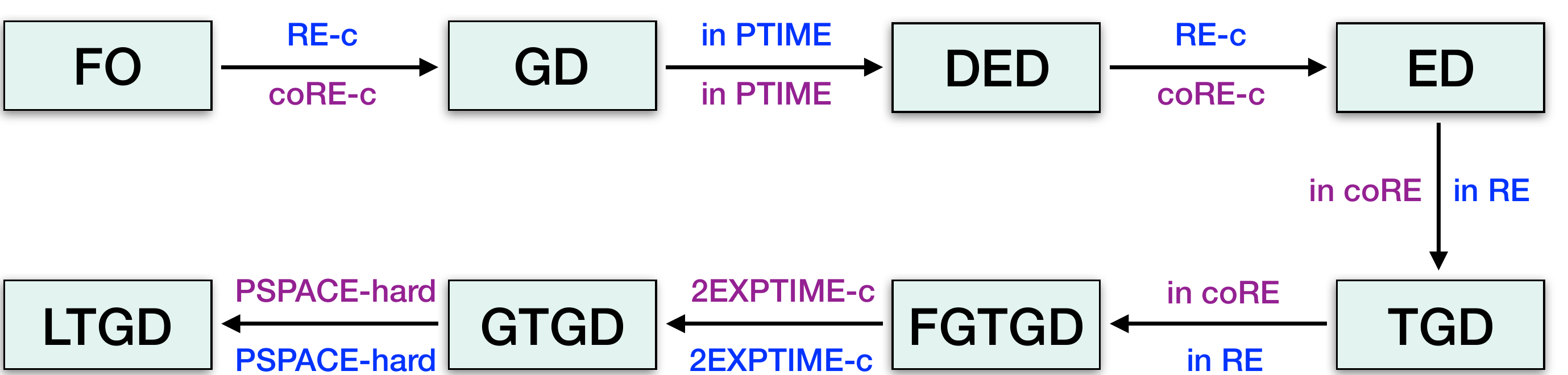}
\end{center}\vspace{-.25cm}
\caption{Complexity of Rewritability}\label{fig:rewritability}\vspace{-.25cm}
\end{figure}

\begin{theorem}\label{thm:complexity_rewritability}
The complexity of rewritability for the above existential rule languages is illustrated in Figure~\ref{fig:rewritability}, where, along each arrow, the bound colored in blue indicates the complexity over arbitrary structures, and the bound colored in purple indicates the complexity over finite structures. 
\end{theorem}

We only explain the general idea of proving the rewritability of $\mathsf{FGTGD}$ to $\mathsf{GTGD}$ is \textsc{2ExpTime}-complete. By Theorem~\ref{thm:char_tgd2gd}, it is equivalent to prove that the preservation property of $\mathsf{FGTGD}$ under disjoint unions is \textsc{2ExpTime}-complete. The membership can be show by reducing this problem to the satisfiability problem of guarded negation logic, a fragment whose satisfibility is \textsc{2ExpTime}-complete~\cite{BaranyCS2015}; while the hardness can be proved by reducing the query answering of $\mathsf{GTGD}$ to the mentioned problem.

\section{Conclusion and Related Work}

We have established model-theoretic characterizations for several important classes of existential rules. Very interestingly, our characterizations show that the guarded-based notions are exactly captured by union-like preservations. Since union-like preservations can be regarded as modular properties in a certain sense, this work also provides alternative perspective on why guarded-based existential rule languages enjoy good computational properties. We believe this may shed new insight on identifying new tractable languages.

There have been a number of earlier researches that are closely related to ours. \citeauthor{BBC2013}~\shortcite{BBC2013} showed that every TGDs-defined sentence in the guarded negation fragment is definable by frontier-guarded TGDs. Over finite structures, \citeauthor{MakowskyV1986}~\shortcite{MakowskyV1986} established several characterizations for full TGDs (TGDs without existential quantifier) and equality-generating dependencies; \citeauthor{CateFK2015}~\shortcite{CateFK2015} observed that the union preservation captures the definability of TGDs by linear TGDs. In informatation integration, \citeauthor{CateK2010}~\shortcite{CateK2010} proved a number of characterizations for source-to-target TGDs (a class of acyclic TGDs) and its subclasses. 
In description logics, \citeauthor{LutzPW2011}~\shortcite{LutzPW2011} established characterizations of $\mathcal{EL}$ and $DL$-$Lite_{{horn}}$ over arbitrary structures. Note that both of these languages can be regarded as sublanguages of existential rule languages. In addition, it is worth noting that our characterization of TGD-definable EDs by perseverations under strictly-homomorphic images and preimages is inspired by~\cite{CasanovasDJ1996}.

\section*{Acknowledgments}

We are deeply indebted to Professor Carsten Lutz for insightful discussions. Heng Zhang thanks the support of the National Natural Science Foundation of China (NO.61806102), and Guifei Jiang acknowledges the support of the National Natural Science Foundation of China (NO.61806102) and the Fundamental Research Funds for the Central Universities.

\bibliographystyle{named}

\bibliography{Xbib.bib}

{

\onecolumn

\newpage

\appendix
\section{Appendix: Detailed Proofs}

\medskip
\subsection{Proofs of Theorem~\ref{cntexm_retrprsv}}

{\noindent\bf Theorem~\ref{cntexm_retrprsv}.}
There is a first-order sentence which is preserved under globally-homomorphic preimages in the finite but is not equivalent to any finite set of GDs over finite structures.

\begin{proof}
We prove this by constructing an example. This can be done by a slight modification to Gurevich and Shelah's counterexample (see, e.g.,~\cite{Rosen2002}), which is used to show the well-known {\L}o\'{s}-Tarski preservation theorem does not hold over finite structures. Below is the revised example.

Let $\varphi_{1}$ be a universal first-order sentence asserting that $Less$ defines a linear order on the intended domain. Let $\varphi_2$ be a conjunction of universal closures of the following formulas:
\begin{eqnarray}
(Min(x)\rightarrow x=y\vee Less(x,y))\wedge\exists v Min(v),\\
(Max(x)\rightarrow x=y\vee Less(y,x))\wedge\exists v Max(v),\\
Succ(x,y)\rightarrow Less(x,y),\\
Succ(x,y)\wedge Less(x,z)\rightarrow y=z\vee Less(y,z).
\end{eqnarray}
Intuitively, the first two formulas state that there are minimum and maximum elements in the domain, and the last two formulas state that $Succ$ defines the relation of immediate successors w.r.t. $Less$ if it exists. 
Let $\varphi_{3}$ denote the sentence
\begin{equation}
\forall x(Max(x)\vee\exists y\, Succ(x,y))
\end{equation}
which asserts that, except the maximum element, every element in the domain has a successor.  
Finally, let $\varphi$ denote $\varphi_1\wedge\varphi_2\wedge(\varphi_3\rightarrow\exists x\neg Q(x))$, where $Q$ is a fresh unary relation symbol. To show the desired proposition, it suffices to prove two claims as follows:

\medskip
{\noindent\em Claim 1.} 
$\varphi$ is preserved under globally-homomorphic preimages in the finite. 

\begin{proof}
Let $\mathcal{A}$ be a finite model of $\varphi$, and $\mathcal{B}$ a structure globally-homomorphic to $\mathcal{A}$. We want to show that $\mathcal{B}$ is also a model of $\varphi$. Clearly, in structure $\mathcal{A}$, $Less$ must be interpreted as a strict linear order on $A$, $Max$ and $Min$ as sets consisting of the maximum and the minimum elements in $A$ w.r.t. $Less^\mathcal{A}$, respectively; and in particular, $Succ^\mathcal{A}$ should be the immediate successor relation on $A$ w.r.t. $Less^\mathcal{A}$ since $\mathcal{A}$ is finite. On the other hand, as $\mathcal{B}$ is a globally-homomorphic preimages of $\mathcal{A}$ and $\mathcal{A}$ is finite, we know
\begin{enumerate}
\item there is an isomorphism $p$ from $\mathcal{B}$ to a substructure of $\mathcal{A}$, and
\item there is a homomorphism $h$ from $\mathcal{A}$ to $\mathcal{B}$ such that $h(p(b))=b$ for all constants $b\in B$.
\end{enumerate}
Next we show that $h$ is injective. Towards a contradiction, assume it is not ture, i.e., there are two distinct constants $a_1,a_2\in A$ such that $h(a_1)=h(a_2)$. Since $Less^\mathcal{A}$ is a strict linear order, either $(a_1,a_2)\in Less^\mathcal{A}$ or $(a_2,a_1)\in Less^\mathcal{A}$ must be true. Due to the symmetry, we only consider the former, i.e., $(a_1,a_2)\in Less^\mathcal{A}$. As $h$ is a homomorphism from $\mathcal{A}$ to $\mathcal{B}$, we have $(h(a_1),h(a_2))\in Less^\mathcal{B}$. Since $p$ is an isomorphism from $\mathcal{B}$ to some substructure of $\mathcal{A}$, we conclude that $(p(h(a_1)),p(h(a_2)))\in Less^\mathcal{A}$. But this is impossible because we have $h(a_1)=h(a_2)$.

By the injectivity of both $p$ and $h$ we conclude $|A|=|B|$. Since $A$ is finite, we know that $\mathcal{A}$ is isomorphic to $\mathcal{B}$. This means that $\mathcal{B}$ is also a model of $\varphi$, which is as desired. 
\end{proof}

\medskip
{\noindent\em Claim 2.} 
$\varphi$ is not equivalent to any finite set of GDs over finite structures.

\begin{proof}
Towards a contradiction, assume the claim is not true. Let $\Sigma$ be a finite set of GDs equivalent to $\varphi$ over finite structures. Let $n$ denote the number of universal variables of $\Sigma$, and let $A$ denote the set $\{0,\dots,n\}$. Let $\mathcal{A}$ be a structure with the domain $A$ such that all the following hold:
\begin{itemize}
\item $Less^\mathcal{A}=\{(i,j):0\le i<j\le n\}$;
\item $Succ^\mathcal{A}=\{(i,i+1):0\le i<n\}$;
\item $Min^\mathcal{A}=\{0\}$, $Max^\mathcal{A}=\{n\}$, and $Q^\mathcal{A}=A$.
\end{itemize}
Clearly, $\mathcal{A}$ is not a model of $\varphi$, and not a model of $\Sigma$ consequently. Thus, there is a GD $\sigma\in\Sigma$ of the form 
\begin{equation}
\phi(\vec{x})\rightarrow\psi(\vec{y})
\end{equation}
and a tuple $\vec{a}$ on $\mathcal{A}$ such that $\mathcal{A}\models\phi[\vec{a}/\vec{x}]$ and $\mathcal{A}\models\neg\psi[\vec{a}/\vec{x}]$, where $\vec{y}\subseteq\vec{x}$, $\phi$ is quantifier-free, and $\psi$ is a formula in $\exists^+$. Take $b$ as a constant  in $A$ but not occurring in $\vec{a}$. Let $\mathcal{B}$ be a structure obtained from $\mathcal{A}$ by removing the fact $Q(b)$. It is easy to see that $\mathcal{B}\models\phi[\vec{a}/\vec{x}]$ and $\mathcal{B}\models\neg\psi[\vec{a}/\vec{x}]$. Thus, $\mathcal{B}$ is also not a model of $\Sigma$. On the other hand, it is trivial to check that $\mathcal{B}$ is in fact a model of $\varphi$, a contradiction as desired.
\end{proof}

With these lemmas, we immediately have the proposition.
\end{proof}

\subsection{Proof of Theorem~\ref{prop:dgd_dp_prsv}}

{\noindent\bf Theorem~\ref{prop:dgd_dp_prsv}.}
A finite set of DEDs is equivalent to a finite set of EDs [over finite structures] iff 
it is preserved under direct products [in the finite]. 
\medskip

To prove the theorem, we need some notations and a lemma.

By a {\em conjunctive query with equality} (CQ$^=$) we mean a first-order formula of the form $\exists\vec{y}\vartheta(\vec{x},\vec{y})$ where $\vartheta$ is a conjunction of atomic formulas (so equalities are allowed).

Suppose $\mathcal{A}_1,\dots,\mathcal{A}_n$ are a sequence of structures over a same signature, we let $\prod_{i=1}^n\mathcal{A}_i$ denote the structure 
\begin{equation}
(\cdots((\mathcal{A}_1\times\mathcal{A}_2)\times\mathcal{A}_3)\times\cdots\times\mathcal{A}_n).
\end{equation}

For each tuple $\vec{a}$, let $\vec{a}(i)$ denote the $i$-th component of $a$.  

The following lemma can be proved by a routine check.

\begin{lemma}
  \label{prodlemma}
  For all CQ$^=$s $\phi(\vec{x})$ and families $(\mathcal{A}_i,\vec{a}_i), i=1,\dots,n$, where
  all $\vec{a}_i$ are of the same length as $\vec{x}$, we have
  $$
  \prod_{i=1}^n \mathcal{A}_i \models \phi[\vec{a}] \quad
  \text{ iff } \quad
  \forall i = 1,\dots,n: \mathcal{A}_i \models \phi[\vec{a}_i]
  $$
  where $\vec{a}$ is a tuple of tuples such that $\vec{a}(i)=(\vec{a}_1(i),\dots,\vec{a}_n(i))$ for all $i =1,\dots,n$.
\end{lemma}

With the above lemma, we are then able to prove the desired theorem.

\begin{proof}[Proof of Theorem~\ref{prop:dgd_dp_prsv}]
  We only show the case for finite structures. The proof for the arbitrary structures is almost the same.
  The direction of ``only if'' follows from Proposition~\ref{prop:d_dp_prsv}.  For
  the direction of ``if'', let $\Sigma$ be a finite set of DEDs. We want to construct a finite set of EDs which is equivalent to $\Sigma$ over finite structures. Let $
  \gamma$ be a DED. W.l.o.g., we write $\gamma$ as the following form
  \begin{equation}\label{eqn_dtgd}
  \varphi(\vec{x})\rightarrow\psi_1(\vec{x})\vee\cdots\vee\psi_k(\vec{x})
  \end{equation}
  where $\varphi$ and $\psi_i$ are CQ$^=$s. Let $\mathrm{split}(\gamma)$ denote the set of EDs
  \begin{equation}\label{eqn_tgd}
  \varphi(\vec{x})\rightarrow\psi_i(\vec{x})
  \end{equation}
  for all $i:=1,\dots,k$, and let $\mathrm{split}(\Sigma)$ denote the union of $\mathrm{split}(\gamma)$ for all DEDs $\gamma\in\Sigma$. Let
  \begin{equation}
  \mathrm{con}(\Sigma)=\{\gamma\in\mathrm{split}(\Sigma):\Sigma
  \vDash_{\mathrm{fin}} \gamma\}.
  \end{equation} 
  Our task is to show that $\mathrm{con}(\Sigma)$ is equivalent to $\Sigma$ over finite structures. We only need to prove
  $\mathrm{con}(\Sigma)\vDash_{\mathrm{fin}}\Sigma$. The other direction is trivial. Towards a contradiction,
  assume this is not the case. Then $\Sigma$ must contain a DED
  $\gamma$ of form~(\ref{eqn_dtgd})
  such that $\mathrm{con}(\Sigma)\not\vDash_{\mathrm{fin}}\gamma$. Let $\varrho_i$
  denote the corresponding rule of form~(\ref{eqn_tgd}). Then
  $\Sigma\not\vDash_{\mathrm{fin}}\varrho_i$ is true  for all $i=1,\dots,k$ because otherwise we
  have $\mathrm{con}(\Sigma)\vDash_{\mathrm{fin}}\gamma$. Thus for
  each $i$, there is a finite model $\mathcal{A}_i$ of $\Sigma$ and a tuple
  $\vec{a}_i\in A_i$ such that
  $\mathcal{A}_i\models\phi[\vec{a}_i]$ and
  $\mathcal{A}_i\not\models\psi_i[\vec{a}_i]$. Let
  $\mathcal{A}=\prod_{1\le i\le k}\mathcal{A}_i$. Clearly, it is finite. Let $\vec{a}$
  be the tuple in $A$ such that $\vec{a}(i)=\vec{a}_i$ for all
  $i=1,\dots,k$. By Lemma~\ref{prodlemma}, we then have
  $\mathcal{A}\models\phi[\vec{a}]$ and
  $\mathcal{A}\not\models\psi_i[\vec{a}]$ for all $i=1,\dots,k$. On the
  other hand, as $\Sigma$ is preserved under direct products in the finite and each
  $\mathcal{A}_i$ is a model of $\Sigma$, we must have
  $\mathcal{A}\models\gamma$. This implies that
  \begin{equation}
  \mathcal{A}\models(\psi_1(\vec{x})\vee\cdots\vee\psi_k(\vec{x}))[\vec{a}/\vec{x}],
  \end{equation} 
  or equivalently,
  $\mathcal{A}\models\psi_i[\vec{a}]$ for some $i=1,\dots,k$, a
  contradiction.
\end{proof}

\subsection{Proof of Theorem~\ref{thm:strict_hom_prsv}}

{\noindent\bf Theorem~\ref{thm:strict_hom_prsv}.}
A finite set of EDs is equivalent to a finite set of TGDs [over finite structures] iff
it is preserved under both strictly-homomorphic images and preimages [in the finite].

\begin{proof}
The direction of ``only-if" can be showed by a routine check. So, it remains to prove the converse. It suffices to prove both of the following statements:
\begin{enumerate}
\item Every ED preserved under strictly-homomorphic images [in the finite] is equivalent to a finite set of body-equality-free EDs [over finite structures]. 
\item Every ED preserved under strictly-homomorphic preimages [in the finite] is equivalent to a finite set of head-equality-free EDs [over finite structures]. 
\end{enumerate}

Only prove the first statement. The proof of the second one is similar. 

Let $\sigma$ be an ED that is preserved under strictly-homomorphic images [in the finite]. We need to show that $\sigma$ is equivalent to a body-equality-free ED [over finite structures]. 
 W.l.o.g., we assume the body of $\sigma$ contains no equality of any form among $x=y$, $x=c$, $c=y$ and $c=c$, where $x,y$ are variables and $c$ is a constant symbol. Note that equalities of the first three forms can be easily eliminated by substituting. So, the only equalities cannot be simply removed are of the form $c=d$ where $c$ and $d$ are distinct constant symbols.

Let $\Sigma$ be a finite set of EDs satisfying the above assumption, and suppose $\Sigma$ is preserved under strictly-homomorphic images [in the finite]. Let $\Sigma_0$ denote the set of all body-equality-free EDs in $\Sigma$. Now let us show that $\Sigma_0$ is equivalent to $\Sigma$ [over finite structures].
It is trivial that $\Sigma\vDash\sigma$ for all $\sigma\in\Sigma_0$. For the converse, to obtain a contradiction, we assume there is an ED $\sigma\in\Sigma$ such that $\Sigma_0\vDash\sigma$ fails. That is, there is a [finite] model $\mathcal{A}$ of $\Sigma_0$ such that $\mathcal{A}\models\neg\sigma$. 
Let $\tau$ denote the signature of $\Sigma$ and $C$ denote the set of constant symbols in $\tau$. Let $(\hat{c})_{c\in C}$ be a sequence of pairwise distinct fresh constants that have no occurrences in $A$. Let $B=A\cup \{\hat{c}:c\in C\}$ and let $h:B\rightarrow A$ be a function with $h(a)=a$ for all $a\in A$ and $h(\hat{c})= c^\mathcal{A}$ for all $c\in C$. Note that $B$ is finite whenever $A$ is finite. Let $\mathcal{B}$ be a $\tau$-structure with domain $B$ such that
\begin{itemize}
\item $c^\mathcal{B}=\hat{c}$ for any constant symbol $c\in\tau$, and 
\item $\vec{b}\in R^\mathcal{B}$ iff $h(\vec{b})\in R^\mathcal{A}$ for any relation symbol $R\in\tau$ and any tuple $\vec{b}$ of constants in $B$ with a proper length.
\end{itemize}
It is clear that $h$ is a strict homomorphism from $\mathcal{B}$ onto $\mathcal{A}$.
We first show that $\mathcal{B}$ is a model of $\Sigma$. Take $\sigma_0\in\Sigma$. If $\sigma_0\in\Sigma_0$, it is easy to see that $\mathcal{B}\models\sigma_0$ since $\mathcal{A}$ is a model of $\Sigma$. Note that, by a routine check, it is not difficult to show that every body-equality-free ED is preserved under strictly-homomorphic images. For the other case, we have $\sigma_0\in\Sigma\setminus\Sigma_0$, i.e., the body of $\sigma_0$ must contains at lease one equality $c=d$ where $c,d$ are distinct constant symbols. Thus, there is no assignment in $\mathcal{B}$ that satisfies the body of $\sigma_0$, which implies that $\mathcal{B}\models\sigma_0$. Therefore, $\mathcal{B}$ is indeed a model of $\Sigma$.

On the other hand, we know that $\mathcal{A}$ is a strictly-homomorphic image of $\mathcal{B}$. Since $\Sigma$ is preserved under strictly-homomorphic images and $\mathcal{A}$ is not a model of $\Sigma$, $\mathcal{B}$ cannot be a model of $\Sigma$, a contradiction as desired. Thus, $\Sigma_0$ is equivalent to $\Sigma$ [over finite structures]. As $\Sigma_0$ is already a set of body-equality-free embedded dependencies, we then obtain the desired statement.
\end{proof}

\subsection{Proofs of Theorem~\ref{thm:char_d2fgd}}

{\noindent\bf Theorem~\ref{thm:char_d2fgd}.}
A finite set of TGDs is equivalent to a finite set of frontier-guarded TGDs [over finite structures] iff it is preserved under isomorphic unions [in the finite]. 
\medskip

A sketched proof of this theorem is presented in the submitted paper, we only present the omitted proofs here. 

\medskip
{\noindent\bf Lemma~\ref{lem:cq_copyprsv}.}
Let $\phi(\vec{x})$ be a CQ, $\tau$ the signature $\tau$ of $\phi$, $\mathcal{A}$ a $\tau$-structure, $\vec{a}$ a tuple on $\mathcal{A}$ with $|\vec{a}|=|\vec{x}|$, and $\mathbb{G}$ a finite set of guarded sets of $\mathcal{A}$ such that every constant in $\vec{a}$ belongs to some $X\in\mathbb{G}$. If $\bigcup_{X\in\mathbb{G}}\mathcal{A}_{X}\models\phi[\vec{a}]$ then $\mathcal{A}\models\phi[\vec{a}]$.

\begin{proof}
Let $\mathcal{B}$ denote $\bigcup_{X\in\mathbb{G}}\mathcal{A}_{X}$ and assume $\mathcal{B}\models\phi[\vec{a}]$. We want to show that $\mathcal{A}\models\varphi[\vec{a}]$. Suppose $\phi(\vec{x})$ is of the form $\exists\vec{y}\vartheta(\vec{x},\vec{y})$ where $\vartheta$ is a conjunction of relational atomic formulas. Clearly, there is a tuple $\vec{b}$ on $\mathcal{B}$ such that $\mathcal{B}\models\vartheta[\vec{a},\vec{b}]$. Let $p$ be a function that maps each element $b\in B$ to $a\in A$ if $b$ is the copy of $a$ in $\mathcal{A}_{X}$ for some $X\in\mathbb{G}$. According to the definition of $\mathcal{B}$, it is easy to see that $p$ is well-defined, and $p|_{A_X}$ is an isomorphism from $\mathcal{A}_X$ to $\mathcal{A}$. Let $\alpha$ be a conjunct of $\vartheta(\vec{a},\vec{b})$. Then $\alpha$ must be of the form $R(\vec{c})$, and $\vec{c}\in R^\mathcal{B}$, By definition, there is some $X\in\mathbb{G}$ such that $\vec{c}\in R^{\mathcal{A}_X}$. From it we conclude that $p(\vec{c})\in R^\mathcal{A}$. We thus have $\mathcal{A}\models\vartheta[\vec{a},p(\vec{b})]$. (Note that $p(\vec{a})=\vec{a}$.) This yields $\mathcal{A}\models\phi[\vec{a}]$ as desired.
\end{proof}

\medskip
\noindent{\em Claim 1.} $\Sigma$ is equivalent to a finite set of diverse dependencies. 

\begin{proof}
Let $\sigma$ be an arbitrary TGD of the form
\begin{equation}
\forall\vec{x}(\phi(\vec{x})\rightarrow\exists\vec{z}\psi(\vec{y},\vec{z}))
\end{equation}
where each variable that appears in $\vec{y}$ also appears in $\vec{x}$, $\phi$ and $\psi$ are conjunctions of relational atomic formulas. Clearly, $\sigma$ is equivalent to the sentence $\sigma_0$, defined as follows:
$$\forall\vec{x}\left[\phi(\vec{x})\wedge\bigwedge_{1\le i<j\le k}(t_i=t_j\vee\neg\, t_i=t_j)\rightarrow\exists\vec{z}\psi(\vec{y},\vec{z})\right],$$
where $t_1,\dots,t_k$ is an enumeration (without repetition) of all terms occurring in $\sigma$.
By applying the distributive law and eliminating equalities from the body, it is not difficult to see that $\sigma_0$ is equivalent to a finite set of diverse TGDs. Note that every positive occurrence of an equality involving only variables in the body can be eliminated by substitutions.
\end{proof}

\medskip
\noindent{\em Claim 2.} $\Sigma\vDash\gamma^\ast$ for all $\gamma\in\Gamma$. 

\begin{proof}
Let $\mathcal{A}$ be a model of $\Sigma$ and $\vec{a}$ a tuple on $\mathcal{A}$ such that $\mathcal{A}\models\lambda_{\mathrm{una}}[\vec{a}]$ and $\mathcal{A}\models\phi[\vec{a}]$. Let $C$ be the set of all interpretations of constant symbols in $\mathcal{A}$. Let $\mathbb{G}$ be the set of guarded sets of $\mathcal{A}$ such that if $X\in\mathbb{G}$ then all constants in $X\setminus C$ co-occur in an atomic formula in $\phi(\vec{a})$. Let $\mathcal{B}=\bigcup_{X\in\mathbb{G}}\mathcal{A}_{X}$. By definition we know $\mathcal{B}\models\lambda_{\mathrm{una}}[\vec{a}]$ and $\mathcal{B}\models\phi[\vec{a}]$. As $\Sigma$ is preserved under isomorphic unions, $\mathcal{B}$ must be a model of $\Sigma$. Consequently, $\mathcal{B}$ is a model of $\gamma$. We thus have that $\mathcal{B}\models\exists\vec{y}\psi[\vec{a}/\vec{x}]$, i.e., there is a tuple $\vec{b}$ on $\mathcal{B}$ with $\mathcal{B}\models\psi[\vec{a},\vec{b}]$. 

Define a substitution $s$ as follows: Given $i=1,\dots,|\vec{y}|$, let $s(y_i)=c$ if for some constant symbol $c$ with $b_i=c^\mathcal{A}$; if no such $c$ then let $s(y_i)=x_j$ for some $j$ with $b_i=a_j$; if no such $j$ either then let $s(y_i)=y_i$, where $a_i,b_i,x_i,y_i$ denote the $i$-th components of $\vec{a},\vec{b},\vec{x},\vec{y}$, respectively. Clearly $\mathcal{B}\models s(\exists\vec{y}\psi)[\vec{a}/\vec{x}]$. By Lemma~\ref{lem:cq_copyprsv}, we have $\mathcal{A}\models s(\exists\vec{y}\psi)[\vec{a}/\vec{x}]$. To complete the proof, it remains to prove $s\in S_\gamma$, i.e., $s(\gamma)$ is quasi-frontier-guarded. Let $\delta(\vec{x},\vec{y})$ be a connected component of the graph of $s(\gamma)$. It suffices to show that all variables occurring in both $\delta$ and $\vec{x}$ co-occur in an atomic formula in $\phi$. Let $x_i,x_j$ be any two distinct variables that occur in both $\vec{x}$ and $\delta$. Since $\delta$ is connected and $\mathcal{B}\models\delta[\vec{a},\vec{b}]$, there must exist some set $X\in\mathbb{G}$ such that $\mathcal{A}_X\models\delta[\vec{a},\vec{b}]$, which implies $\{a_i,a_j\}\subseteq A_{X}$ immediately. Since $A_{X}\cap A=X$, we know $\{a_i,a_j\}\subseteq X$. As $\mathcal{A}\models\lambda_{\mathrm{una}}[\vec{a}]$, there must be a bijection from $\vec{a}$ to $\vec{x}$, and neither $a_i$ nor $a_j$ is the interpretation of any constant symbol. By the definition of $\mathbb{G}$, there is then a conjunct $\alpha$ of $\phi$ such that  $x_i$ and $x_j$ co-occur in $\alpha$. This yields the claim.
\end{proof}

\noindent{\em Claim 3.} $\mathrm{con}(\Gamma^\dag)$ is equivalent to $\Gamma^\dag$. 

\begin{proof}
The main idea of proving this claim is as follows. Since $\Sigma$ is preserved under direct products and strictly-homomorphic preimages, $\Gamma^\dag$ is also preserved under direct products and strictly-homomorphic preimages. By a direct product argument used in Theorem~\ref{prop:dgd_dp_prsv}, one can prove that $\Gamma^\dag$ is equivalent to $\Gamma_0$, the union of $\mathrm{con}(\Gamma^\dag)$ and a set of equality-generating dependencies (EGDs), where each EGD is a body-equality-free ED whose head is a single equality involving no existential variable. Next, by a strictly-homomorphic preimage preservation argument used in Theorem~\ref{thm:strict_hom_prsv}, one can show that all EGDs in $\Gamma_0$ can be eliminated, which proves Claim 3.
\end{proof}

\subsection{Proof of Theorem~\ref{thm:char_tgd2gd}}

{\noindent\bf Theorem~\ref{thm:char_tgd2gd}.}
A finite set of TGDs is equivalent to a finite set of guarded TGDs [over finite structures]
iff it is preserved under disjoint unions [in the finite]. 

\begin{proof}
The direction of ``only-if" follows from Proposition~\ref{prop:gd_duprv}.For the converse, we only consider finite structures. The case of arbitrary structures can be proved in a similar way by the finite model property of frontier-guarded TGDs.

The general idea of proving the desired direction is as follows: We first show that every finite set of frontier-guarded TGDs preserved under disjoint unions in the finite is equivalent to a finite set of guarded TGDs over finite structures. As the disjoint union preservation in the finite always implies the isomorphic union preservation in the finite, by Theorem~\ref{thm:char_d2fgd}, we then have the desired conclusion. Now we prove that every finite set of frontier-guarded TGDs preserved under disjoint unions in the finite is equivalent to a finite set of guarded TGDs over finite structures.

Before presenting the proof, we first ndefine some notions. For a technical reason, we redefine
{\em CQs} (over $\tau$) as first-order formulas built on relational atomic formulas (over $\tau$) by using connective $\wedge$ and quantifier $\exists$ only. Note that every CQ can be equivalently rewritten as a formula of the form $\exists\vec{x}\vartheta(\vec{x})$, where $\vartheta$ is a conjunction of atomic formulas. It is clear that every TGD $\sigma$ can be equivalently written as a sentence of the form 
\begin{equation}\label{eqn:tgd_normform}
\forall\vec{x}(\phi(\vec{x})\rightarrow\psi(\vec{x}))
\end{equation}
where $\phi$ and $\psi$ are CQs with free variables exactly among $\vec{x}$. Such a sentence is called a {\em canonical form} of $\sigma$. The {\em rank} of a CQ is the maximum nesting depth of quantifier $\exists$ in it; the {\em rank} of a sentence of the form~(\ref{eqn:tgd_normform}) is the maximum rank of $\phi$ and $\psi$. Moreover, the {\em rank} of a TGD is the minimum rank of all canonical forms of $\sigma$.

Now, let us prove the direction of ``if".
Let $\Sigma$ be a finite set of frontier-guarded TGDs that is preserved under disjoint unions. Let $\tau$ be the signature of $\Sigma$ and $k$ the maximum rank of TGDs in $\Sigma$. Let $\mathrm{con}(\Sigma)$ be the set of guarded TGDs $\sigma$ of rank $\le k$ such that $\Sigma\vDash_{\mathrm{fin}}\sigma$. Up to logical equivalence, it is clear that there are only a finite number of guarded TGDs in $\mathrm{con}(\Sigma)$. Now, we show that $\Sigma$ is equivalent to $\mathrm{con}(\Sigma)$ over finite structures. The direction $\Sigma\vDash_{\mathrm{fin}}\mathrm{con}(\Sigma)$ is trivial. We need only to show the converse, i.e., $\mathrm{con}(\Sigma)\vDash_{\mathrm{fin}}\sigma$ for all $\sigma\in\Sigma$. Let $\mathcal{A}$ be a finite model of $\mathrm{con}(\Sigma)$. We have to show that $\mathcal{A}$ is also a model of $\Sigma$.

Let $m$ denote the maximum airty of relation symbols from $\tau$. Fix $\vec{v}$ as a tuple of pairwise distinct variables $v_1,\dots,v_m$. Let $\mathbb{V}$ denote the set of all assignments ${s}:\vec{u}\rightarrow A$ such that $\vec{u}$ is a subtuple of $\vec{v}$ and that $\mathcal{A}\models\alpha[{s}]$ for some atomic formula $\alpha(\vec{v})$. Note that, since both $A$ and $m$ are finite, the number of such assignments should also be finite.
Fix ${s}:\vec{u}\rightarrow A$ as an assignment in $\mathbb{V}$ with $\vec{u}\subseteq\vec{v}$. Let $\varphi_{{s}}(\vec{u})$ be a conjunction of all atomic formulas $\alpha$ over $\tau$, involving only variables among $\vec{u}$ as arguments (not necessary all of them), such that $\mathcal{A}\models\alpha[{s}]$, and let $\psi_{{s}}(\vec{u})$ be a disjunction of all CQs $\vartheta(\vec{u})$ over $\tau$ of rank $\le k$ such that $\mathcal{A}\models\neg\vartheta[{s}]$. Let $\gamma_{{s}}$ denote the sentence $\exists\vec{u}(\varphi_{{s}}(\vec{u})\wedge\neg\psi_{{s}}(\vec{u}))$.
Then we have the property:

\medskip
\noindent{\em Claim 1.} For all ${s}\in\mathbb{V}$, $\Sigma\cup\{\gamma_{{s}}\}$ has a finite model.

\begin{proof} Towards a contradiction, let us assume that $\Sigma\cup\{\gamma_{{s}}\}$ has no finite model, or equivalently, $\Sigma\vDash_{\mathrm{fin}}\neg\gamma_{{s}}$. It is clear that $\neg\gamma_{{s}}$ is equivalent to some guarded TGD $\sigma$ over $\tau$ of rank $\le k$. This means that $\sigma\in\mathrm{con}(\Sigma)$. Thus, $\mathcal{A}$ is a model of $\sigma$, or equivalently, a model of $\neg\gamma_{{s}}$. This implies that either $\mathcal{A}\models\phi_{s}[{s}]$ or $\mathcal{A}\models\neg\psi_{{s}}[{s}]$ is false, a contradiction as desired.
\end{proof}

For all ${s}\in\mathbb{V}$, let $\mathcal{B}_{{s}}$ be a finite model of $\Sigma\cup\{\gamma_{{s}}\}$. Without loss of generality, we assume all of the following are true:
\begin{enumerate}
\item the intersection of domains of $\mathcal{A}$ and $\mathcal{B}_{{s}}$, $A\cap B_{{s}}$, is $ran({s})$, where $ran({s})=\{{s}(v):\exists v\text{ s.t. }{s}(v)\text{ is defined}\}$;
\item the intersection of domains of $\mathcal{B}_{{s}_1}$ and $\mathcal{B}_{{s}_2}$ is the intersection of $ran({s}_1)$ and $ran({s}_2)$ for each pair of distinct assignments ${s}_1$ and ${s}_2$ from $\mathbb{V}$;
\item both $\mathcal{B}_{s}\models\varphi_{s}[{s}]$ and $\mathcal{B}_{s}\models\neg\psi_{s}[{s}]$ are true.
\end{enumerate}
From the definition, it is not difficult to check that, given any two assignments ${s}_1,{s}_2\in\mathbb{V}$, if $ran({s}_1)\cap ran({s}_2)$ is not empty we then have $\mathcal{B}_{{s}_1}|_{ran({s}_1)}=\mathcal{B}_{{s}_2}|_{ran({s}_2)}$, which is one of the preconditions to apply the disjoint union preservation.

Furthermore, let $\mathcal{B}$ denote the union of $\mathcal{B}_{{s}}$ for all assignments ${s}\in\mathbb{V}$. 
Note that, since both $\mathbb{V}$ and $\mathcal{A}$ are finite, $\mathcal{B}$ should be finite, too. Now one can prove the following property:

\medskip
\noindent{\em Claim 2.} For all ${s}\in\mathbb{V}$ and all CQs $\phi(\vec{u})$ over $\tau$ of rank $\le k$, we have $\mathcal{A}\models\phi[{s}]$ iff $\mathcal{B}\models\phi[{s}]$.

\begin{proof} Let ${s}$ be an assignment from $\mathbb{V}$ and $\phi(\vec{x})$ be a CQ over $\tau$ of rank $\le k$. We first consider the direction ``only-if". As mentioned previously, we have that $\mathcal{B}_{{s}}|_{ran({s})}=\mathcal{A}|_{ran({s})}$. From this, it is easy to see that $\mathcal{B}|_A=\mathcal{A}$, which implies that the identity function on $A$ is a homomorphism from $\mathcal{A}$ to $\mathcal{B}$. By the well-known fact that CQs are preserved under homomorphisms, from $\mathcal{A}\models\phi[{s}]$ we can conclude $\mathcal{B}\models\phi[{s}]$. 

\smallskip
For the converse, let us assume $\mathcal{B}\models\phi[{s}]$. We need to prove $\mathcal{A}\models \phi[{s}]$. Suppose $\phi(\vec{u})$ is of the form $\exists\vec{y}\vartheta(\vec{u},\vec{y})$ where, without loss of generality, $\vartheta(\vec{u},\vec{y})$ is assumed to be a conjunction of atomic formulas $\alpha_1$, \dots, $\alpha_n$. It is clear that there exists an assignment ${s}^+:\vec{u}\cup\vec{y}\rightarrow B$, which extends ${s}$, such that $\mathcal{B}\models\vartheta[{s}^+]$. By the definition of $\mathcal{B}$, for each $i\in\{1\le i\le n\}$ there is an assignment ${t}_i\in\mathbb{V}$ such that $\mathcal{B}_{{t}_i}\models\alpha_i[{s}^+]$. Let $S$ be the set $\{{t}_1,\dots,{t}_n\}$. Let $\vec{y}_0$ denote the tuple of variables $y\in\vec{y}$ with ${s}^+(y)\in A$. For each assignment ${t}\in S$, let $\vec{y}_{{t}}$ denote the tuple of variables $y\in\vec{y}$ with ${s}^+(y)\in B_{t}\setminus A$, and let $\vartheta_{t}$ be the conjunction of atomic formulas $\alpha_i$ from $\theta$ in which any variable belongs to $\vec{u}\cup\vec{y}_0\cup\vec{y}_{t}$. Let ${s}^*$ be the restriction of ${s}^+$ to $\vec{u}\cup\vec{y}_0$. Let $\delta_{{t}}$ denote the CQ $\exists\vec{y}_{{t}}\vartheta_{{t}}$. It is clear that $\mathcal{B}_{{t}}\models\delta_{{t}}[{s}^*]$.  Next, let us prove $\mathcal{A}\models\delta_{t}[{s}^*]$.

Let $V_{t}$ be the set of variables $z\in\vec{u}\cup\vec{y}_0$ that appear in $\delta_{t}$. Let $f:V_{t}\rightarrow ran({t})$ be a function such that ${t}(f(z))={s}^*(z)$ for all variables $z\in V_{t}$. From $\mathcal{B}_{{t}}\models\delta_{{t}}[{s}^*]$, we know that such a function exists. By definition, it is also clear that $\mathcal{A}\models\delta_{t}[{s}^*]$ iff $\mathcal{A}\models f(\delta_{t})[{t}]$. So, it suffices to show $\mathcal{A}\models f(\delta_{t})[{t}]$. Towards a contradiction, assume this is not the case, i.e., $\mathcal{A}\models \neg f(\delta_{t})[{t}]$. By the definition of $\psi_{t}$, $f(\delta_{t})$ must be a disjunct of $\psi_{t}$. By assumption 3, it holds that $\mathcal{B}_{t}\models\neg\psi_{t}[{t}]$. Thus, we have $\mathcal{B}_{t}\models\neg f(\delta_{t})[{t}]$. By the definition of $f$, this means $\mathcal{B}_{t}\models\neg\delta_{t}[{s}^*]$, a contradiction.

Now let us continue the proof of Claim 2. By the definition of $\vec{y}_{t}$, it is not difficult to see that, for each pair of distinct assignments ${t},{t}'\in S$, $\vec{y}_{{t}}$ and $\vec{y}_{{t}'}$ is disjoint. (Recall that $B_{t}\cap B_{{t}'}=ran({t})\cap ran({t}')\subseteq A$.) This means that $\phi(\vec{u})$ is equivalent to the CQ $\exists\vec{y}_0\bigwedge_{{t}\in S}\delta_{{t}}$. Since $\mathcal{A}\models\delta_{{t}}[{s}^*]$ for all ${t}\in S$, we then have $\mathcal{A}\models\phi[{s}^*]$, or equivalently, $\mathcal{A}\models\phi[{s}]$ as desired.
\end{proof}

\medskip
Now, we are in the position to complete the proof. By assumption, for every assignment ${s}\in\mathbb{V}$, $\mathcal{B}_{{s}}$ is a model of $\Sigma$. Since $\Sigma$ is preserved under disjoint unions in the finite, $\mathcal{B}$ should be also a model of $\Sigma$. Let $\sigma\in\Sigma$. W.l.o.g., we assume that $\sigma$ is of the form $\forall\vec{u}(\alpha(\vec{u})\wedge\phi(\vec{u})\rightarrow\psi(\vec{u})),$
where $\alpha$ is the guard of $\sigma$; both $\phi$ and $\psi$ are CQs of ranks $\le k$. Assume ${s}:\vec{u}\rightarrow A$ is an assignment such that $\mathcal{A}\models\alpha[{s}]$ and $\mathcal{A}\models\phi[{s}]$. It is clear that ${s}\in\mathbb{V}$. By Claim 2, we then have both $\mathcal{B}\models\alpha[{s}]$ and $\mathcal{B}\models\phi[{s}]$. Since $\mathcal{B}$ is a model of $\sigma$, it must hold that $\mathcal{B}\models\psi[{s}]$. Again by Claim 2, we obtain $\mathcal{A}\models\psi[{s}]$, which completes the proof.
\end{proof}

\subsection{Proof of Theorem~\ref{thm:char_tgd2ld}}

{\noindent\bf Theorem~\ref{thm:char_tgd2ld}.}
A finite set of TGDs is equivalent to a finite set of linear TGDs iff it is preserved under unions.

\medskip
To prove this, we first give a brief review of the guarded negation fragment, a decidable fragment of first-order logic recently proposed in~\cite{BaranyCS2015}. Given a signature $\tau$, formulas in {\em guarded negation fragment} of first-order logic (GNFO) over $\tau$ are defined as follows:
\begin{equation}
\varphi::=\alpha\mid\exists x\varphi\mid\varphi\wedge\varphi\mid\varphi\vee\varphi\mid\alpha(\vec{x},\vec{y})\wedge\neg\varphi(\vec{x})
\end{equation}
where $\alpha$ is an atomic formula over $\tau$ and $\vec{x},\vec{y}$ are tuples of variables. It had been proved in~\cite{BaranyCS2015} GNFO enjoys the finite model property, i.e., every sentence in GNFO has a model iff it has a finite one.

\begin{proof}[Proof of Theorem~\ref{thm:char_tgd2ld}]
The direction of ``only-if" can be proved by a routine check. It thus remains to prove the converse. 
The general idea is as follows: We first show that every finite set of frontier-guarded TGDs preserved under unions is equivalent to a finite set of linear TGDs. As the union preservation always implies the isomorphic union preservation, by Theorem~\ref{thm:char_d2fgd}, we then have the desired conclusion. Now we prove that every finite set of frontier-guarded TGDs preserved under unions is equivalent to a finite set of linear TGDs.

Let $\Sigma$ be a finite set of frontier-guarded TGDs that is preserved under unions. Let $\tau$ be the signature of $\Sigma$ and $k$ be the maximum rank of CQs that appear in $\Sigma$. Let $\mathrm{con}(\Sigma)$ be the set of linear TGDs $\sigma$ of rank $\le k$ such that $\Sigma\vDash\sigma$. Up to logical equivalence, there are only a finite number of linear TGDs in $\mathrm{con}(\Sigma)$. This means that $\mathrm{con}(\Sigma)$ can be regarded as a sentence in GNFO. Now, we want to show that $\Sigma$ is equivalent to $\mathrm{con}(\Sigma)$. The direction $\Sigma\vDash\mathrm{con}(\Sigma)$ is trivial. So it remains to show the converse, i.e., $\mathrm{con}(\Sigma)\vDash\sigma$ for all $\sigma\in\Sigma$. By the finite model property of GNFO, it suffices to prove $\mathrm{con}(\Sigma)\vDash_{\mathrm{fin}}\sigma$ for all $\gamma\in\Sigma$. 
Let $\mathcal{A}$ be a finite model of $\mathrm{con}(\Sigma)$. We have to show that $\mathcal{A}$ is also a model of $\Sigma$.

Let $m$ be the maximum airty of relation symbols from $\tau$. Fix $\vec{v}$ as a tuple of pairwise distinct variables $v_1,\dots,v_m$. Let $\mathbb{G}$ denote the set of all pairs $\varpi=(\alpha(\vec{u}),{s})$ where $\vec{u}$ is a subtuple of $\vec{v}$, $\alpha(\vec{v})$ is a relational atomic formula over $\tau$ and ${s}:\vec{v}\rightarrow A$ is an assignment in $\mathcal{A}$ with $\mathcal{A}\models\alpha[{s}]$ for some atomic formula $\alpha(\vec{v}_0)$ with $\vec{v}_0\subseteq\vec{v}$. As both $A$ and $m$ are finite, the set $\mathbb{G}$ is finite, too.
Fix $\varpi$ as a pair $(\alpha(\vec{u}),{s})$ from $\mathbb{G}$. Let $\psi_{{s}}(\vec{u})$ denote the disjunction of CQs $\vartheta(\vec{u})$ over $\tau$ of rank $\le k$ such that $\mathcal{A}\models\neg\vartheta[{s}]$. Let
$
\gamma_{\varpi}$ be the sentence $\exists\vec{u}(\alpha(\vec{u})\wedge\neg\psi_{{s}}(\vec{u})).
$
To carry out the proof, we need two claims:

\medskip
\noindent{\em Claim 1.} $\Sigma\cup\{\gamma_{\varpi}\}$ is satisfiable.

\begin{proof} Towards a contradiction, let us assume $\Sigma\cup\{\gamma_{\varpi}\}$ is unsatisfiable, or equivalently, $\Sigma\vDash\neg\gamma_{\varpi}$. Clearly, $\neg\gamma_{\varpi}$ is equivalent to some linear TGD $\sigma$ over $\tau$ of rank $\le k$. We then have $\sigma\in\mathrm{con}(\Sigma)$. Thus, $\mathcal{A}$ is a model of $\sigma$, or equivalently, a model of $\neg\gamma_{\varpi}$. Consequently, either $\mathcal{A}\models\alpha[{s}]$ or $\mathcal{A}\models\neg\psi_{{s}}[{s}]$ is false, a contradiction.
\end{proof}

For each pair $\varpi\in\mathbb{G}$, take $\mathcal{B}_{\varpi}$ as a model of $\Sigma\cup\{\gamma_{\varpi}\}$. Without loss of generality, we assume all of the following:
\begin{enumerate}
\item the intersection of domains of $\mathcal{A}$ and $\mathcal{B}_{\varpi}$ is $ran({s})$ (i.e., the range of ${s}$) if $\varpi$ denotes the pair $(\alpha(\vec{u}),{s})$ from $\mathbb{G}$,
\item the intersection of domains of $\mathcal{B}_{\varpi_1}$ and $\mathcal{B}_{\varpi_2}$ is $ran({s}_1)\cap ran({s}_2)$ for any two distinct pairs $\varpi_1=(\alpha_1,{s}_1)$ and $\varpi_2=(\alpha_2,{s}_2)$ from $\mathbb{G}$, and
\item both $\mathcal{B}_\varpi\models\alpha[{s}]$ and $\mathcal{B}_\varpi\models\neg\psi_{s}[{s}]$ are true.
\end{enumerate}

Furthermore, let $\mathcal{B}$ denote the union of $\mathcal{B}_{\kappa}$ for all pairs $\kappa\in\mathbb{G}$. 
Let $\varpi=(\alpha(\vec{x}),{s})$ be a pair from $\mathbb{G}$ where ${s}:\vec{x}\rightarrow A$ is an assignment in $\mathcal{A}$. Then we have the following property:

\medskip
\noindent{\em Claim 2.} $\mathcal{A}\models\phi[{s}]$ iff $\mathcal{B}\models\phi[{s}]$ for all CQs $\phi(\vec{x})$ over $\tau$ of rank $\le k$.

\begin{proof} Let $\phi(\vec{x})$ be a CQ over $\tau$ of rank $\le k$. We first consider the direction of ``only-if". Let $\alpha(\vec{u})$ be any relational atomic formula over $\tau$. Suppose ${t}:\vec{u}\rightarrow A$ is an assignment such that $\mathcal{A}\models\alpha[{t}]$. Let $\kappa=(\alpha,{t})$. It is clear that $\kappa\in\mathbb{G}$. According to Assumption 3, it holds that $\mathcal{B}_{\kappa}\models\alpha[{t}]$, which implies that $\mathcal{B}\models\alpha[{t}]$ immediately. This means that the identity function on $A$ is a homomorphism from $\mathcal{A}$ to $\mathcal{B}$. By the well-known fact that CQs are preserved under homomorphisms, from $\mathcal{A}\models\phi[{s}]$ we then have $\mathcal{B}\models\phi[{s}]$. 

\smallskip
For the converse, let us assume that $\mathcal{B}\models\phi[{s}]$. We need to prove $\mathcal{A}\models \phi[{s}]$. Suppose $\phi(\vec{x})$ is of the form $\exists\vec{y}\vartheta(\vec{x},\vec{y})$ where, without loss of generality, assume that $\vartheta(\vec{x},\vec{y})$ is a conjunction of atomic formulas $\alpha_1$, \dots, $\alpha_n$. It is clear that there exists an assignment ${s}^+:\vec{x}\cup\vec{y}\rightarrow B$, which extends ${s}$, such that $\mathcal{B}\models\vartheta[{s}^+]$. According to the definition of $\mathcal{B}$, there exists a pair $\varpi_i\in\mathbb{G}$ such that $\mathcal{B}_{\varpi_i}\models\alpha_i[{s}^+]$. Let $S$ be the set $\{\varpi_1,\dots, \varpi_n\}$. For each pair $\kappa\in S$, let $\vartheta_\kappa$ be a conjunction of all the atomic formulas $\alpha_i, 1\le i\le n$, such that $\mathcal{B}_\kappa\models\alpha_i[{s}^+]$. Let $\vec{y}_0$ denote the tuple of variables $y\in\vec{y}$ such that ${s}^+(y)$ belongs to the range of ${s}$, and let $\vec{y}_{\kappa}$ denote the tuple of variables appearing in $\vartheta_{\kappa}$ but not in $\vec{x}\cup\vec{y}_0$. Let ${s}^*$ be the restriction of ${s}^+$ to $\vec{x}\cup\vec{y}_0$. By the definition of $\vec{y}_0$, we have $ran({s}^*)\subseteq A$. Let $\delta_{\kappa}$ denote the CQ $\exists\vec{y}_{\kappa}\vartheta_{\kappa}$. It is clear that $\mathcal{B}_{\kappa}\models\delta_{\kappa}[{s}^*]$.  Next, let us prove $\mathcal{A}\models\delta_\kappa[{s}^*]$.

Let $\vec{z}_\kappa$ be the set of variables $z\in\vec{x}\cup\vec{y}_0$ that appear in $\delta_\kappa$. Suppose $\kappa$ is of the form $(\beta(\vec{u}),r)$ where $\vec{u}$ is the domain of the assignment $r$. Let $f:\vec{z}_\kappa\rightarrow\vec{u}$ be a function such that ${t}(f(z))={s}^*(z)$ for all $z\in\vec{z}_\kappa$. From $\mathcal{B}_{\kappa}\models\delta_{\kappa}[{s}^*]$, we know that such a function exists. By definition, it is also clear that $\mathcal{A}\models\delta_\kappa[{s}^*]$ iff $\mathcal{A}\models f(\delta_\kappa)[{t}]$. So, it suffices to show that $\mathcal{A}\models f(\delta_\kappa)[{t}]$. Towards a contradiction, assume this is not true, i.e., $\mathcal{A}\models \neg f(\delta_\kappa)[{t}]$. By the definition of $\gamma_\kappa$, $f(\delta_\kappa)$ must be a disjunct of $\psi_{t}$. By Assumption 3, we have $\mathcal{B}_{\kappa}\models\neg\psi_{t}[{t}]$, which implies $\mathcal{B}_{\kappa}\models\neg f(\delta_\kappa)[{t}]$. By the definition of $f$, we then have $\mathcal{B}_\kappa\models\neg\delta_\kappa[{s}^*]$, a contradiction.

By the above argument, we then have the conclusion $\mathcal{A}\models\delta_\varpi[{s}^*]$. Now let us continue the proof of Claim 2. By the definition of $\vec{y}_\varpi$, it is not difficult to see that, for any two pairs $\varpi,\varpi'\in S$, $\vec{y}_{\varpi}$ and $\vec{y}_{\varpi'}$ must be disjoint. (Recall that $B_\varpi\cap B_{\varpi'}=ran({t})\cap ran({t}')\subseteq A$, where ${t}$ and ${t}'$ are the assignments in $\varpi$ and $\varpi'$, respectively.) This means that $\phi(\vec{x})$ is equivalent to the CQ $\exists\vec{y}_0\bigwedge_{\varpi\in S}\delta_{\varpi}$. Since $\mathcal{A}\models\delta_{\varpi}[{s}^*]$ for all $\varpi\in S$, we then have that $\mathcal{A}\models\phi[{s}^*]$, or equivalently, $\mathcal{A}\models\phi[{s}]$ as desired.
\end{proof}

Now, we are in the position to complete the proof. For each pair $\varpi\in\mathbb{G}$, $\mathcal{B}_{\varpi}$ is clearly a model of $\Sigma$.  Since $\Sigma$ is preserved under disjoint unions, $\mathcal{B}$ should be a model of $\Sigma$, too. Let $\sigma\in\Sigma$. Clearly, $\sigma$ is equivalent to a first-order sentence $\delta$ of the form 
\begin{equation}
\forall\vec{x}(\alpha(\vec{x})\wedge\phi(\vec{x})\rightarrow\psi(\vec{x})),
\end{equation}
where $\alpha$ is a relational atomic formula in which every variable among $\vec{x}$ has at lease one occurrence; both $\phi$ and $\psi$ are CQs of ranks $\le k$. Assume that ${s}:\vec{x}\rightarrow A$ is an assignment such that both $\mathcal{A}\models\alpha[{s}]$ and $\mathcal{A}\models\phi[{s}]$ hold. Let $\varpi$ denote the pair $(\alpha,{s})$. It is clear that $\varpi\in\mathbb{G}$. By Claim 2, we have $\mathcal{B}\models\alpha[{s}]$ and $\mathcal{B}\models\phi[{s}]$. From $\mathcal{B}\models\delta$ we then have $\mathcal{B}\models\psi[{s}]$. Again by Claim 2, it must be true that $\mathcal{A}\models\psi[{s}]$, and this proves the desired theorem.
\end{proof}

\subsection{Proof of Theorem~\ref{thm:complexity_rewritability}}

{\noindent\bf Theorem~\ref{thm:complexity_rewritability}.}
The complexity of rewritability for the above existential rule languages is illustrated in Figure~\ref{fig:rewritability}, where, along each arrow, the bound colored in blue indicates the complexity over arbitrary structures, and the bound colored in purple indicates the complexity over finite structures. 
\medskip

To prove the \textsc{RE}-completeness of the rewritability from $\mathsf{FO}$ to $\mathsf{GD}$ over arbitrary structures, by Theorem~\ref{thm:ghom_prsv}, it suffices to show the following theorem.

\begin{theorem}\label{thm:cmplx_ghom_prsv}
Determining whether
a given first-order sentence is preserved under globally-homomorphic preimages is \textsc{RE}-complete, and the finite model-theoretic version of this problem is \textsc{coRE}-complete.
\end{theorem}

\begin{proof}
(Hardness) It is clear that determining whether a given finite set of GDs is valid [over finite structures] is [\textsc{co}]\textsc{RE}-hard. Note that the query answering problem of TGDs can be easily reduced to this problem. The general idea of proving this theorem is by reducing the former to the desired problem.

Let $\varphi$ be the conjunction of a finite set of GDs. Let $\psi$ denote $\exists x \neg Q(x)$ which, as shown in Example~\ref{exm:ghom_prv_cntexp}, is not preserved under globally-homomorphic preimgages, where $Q$ does not occur in $\varphi$. To yield the hardness, it suffices to show that $\varphi$ is valid [over finite structures] iff $\neg\varphi\wedge\psi$ is preserved under globally-homomorphic preimgages [in the finite]. 

The direction of ``if" is trivial. So we need only to prove the converse. Suppose $\varphi$ has a [finite] countermodel, say $\mathcal{A}$. It suffices to show that $\neg\varphi\wedge\psi$ is not preserved under globally-homomorphic preimgages [in the finite]. Let $c$ be a constant not in $A$. Let $\tau$ and $\tau'$ be signatures of $\varphi$ and $\neg\varphi\wedge\psi$, respectively. Let $\mathcal{B}$ be a $\tau$-structure defined as follows:
\begin{enumerate}
\item the domain $B=A\cup\{c\}$;
\item $c^\mathcal{B}=c^\mathcal{A}$ for all constant symbols $c\in\tau$;
\item $R^\mathcal{B}=R^\mathcal{A}$ for all relation symbols $R\in\tau$.
\end{enumerate}
Clearly, $\mathcal{A}$ is globally-homomorphic to $\mathcal{B}$. By Theorem~\ref{thm:ghom_prsv}, $\mathcal{B}$ must be a model of $\neg\varphi$.
Let $\mathcal{A}'$ and $\mathcal{B}'$ be $\tau$-expansions of $\mathcal{A}$ and $\mathcal{B}$, respectively, which are defined as follows:
$$Q^{\mathcal{A}'}=Q^{\mathcal{B}'}=A.$$
It is easy to see that $\mathcal{B}'$ is a model of $\neg\varphi\wedge\psi$, but $\mathcal{A}'$ is not. On the other hand, one can easily verify that $\mathcal{A}'$ is globally-homomorphic to $\mathcal{B}'$, which implies that $\neg\varphi\wedge\psi$ is not preserved under globally-homomorphic preimages.

\smallskip
(Membership) We first consider the original problem. By Theorem~\ref{thm:ghom_prsv}, to solve this problem, it is equivalent to checking, given a first-order sentence, whether there is a finite set of GDs such that they are equivalent. Let $\Sigma_0,\Sigma_1,\dots$ be an effective enumeration of all sets in $\mathsf{GD}$. It is well-known that, for each $i\ge 0$, there is a deterministic Turing machine, say $M_i$, to decide whether a given first-order sentence is equivalent to $\Sigma_i$. Let $M'$ be a Turing machine such that, given a first-order sentence $\varphi$ as input, $M'$ works as follows:
\begin{enumerate}
\item For $i=0$ to $\infty$ do Steps 2-3.
\item For $j=0$ to $i$ do Step 3.
\item Run the $(i-j)$-th stage of $M_j$ on input $\varphi$ if it is still active; let $M'$ accept $\varphi$ if $M$ accepts at this stage.
\end{enumerate}
It is easy to see that $M'$ accepts exactly the class of first-order sentences that are equivalent to some finite set of GDs, or equivalently, that are preserved under globally-homomorphic preimages, which yields the desired membership.   

Next let us turn to the finite model-theoretic version of the problem. We address the complement of this problem, i.e., determining whether a given first-order sentence is not preserved under globally-homomorphic preimages in the finite. It is easy to see that there is an algorithm, say $A$, to check, given a first-order sentences and an integer $n\ge 1$, whether $\varphi$ is preserved under globally-homomorphic preimages over the class of structures with domain size $\le n$. We thus can define a procedure that, for $n=1$ to $\infty$,  calls $A$ with parameters $\varphi$ and $n$. Clearly, $\varphi$ is not preserved under globally-homomorphic preimages in the finite iff there is some $n$ such that $A$ return false on input $(\varphi,n)$. This yields the \textsc{coRE}-membership.
\end{proof}

As the characterization of GDs by the preservation under globally-homomorphic preimages fails on finite structures, it is not possible to obtain the desired complexity bounds by a straightforward application of the above theorem. However, the \textsc{coRE}-completeness of the rewritability of $\mathsf{FO}$ to $\mathsf{GD}$ can be obtained by an argument that is almost the same as that for the \textsc{coRE}-completeness of determining the preservation under globally-homomorphic preimages in the finite. 

\medskip
To yield the \textsc{PTime}-membership for the rewritability from $\mathsf{GD}$ to $\mathsf{DED}$ [over finite structures], by Proposition~\ref{prop:trivial_sharp}, it suffices to show that checking whether a given set of GDs has both a trivial model and a sharp model is in \textsc{PTime}, which is done by the following proposition:
 
\begin{proposition}\label{prop:ld_union_prsv_up}
Determining whether a given finite set of GDs has both a trivial model and a sharp model is in $\textsc{PTime}$. 
\end{proposition}

\begin{proof}
Let $\Sigma$ be the given finite set of GDs, and let $\tau$ be the signature of $\Sigma$. Let $\mathcal{A}$ be the trivial structure over $\tau$, and $\mathcal{B}$ the sharp structure over $\tau$. By grounding $\Sigma$ by $\mathcal{A}$ (resp., $\mathcal{B}$), we obtain a set $\Sigma_1$ (resp., $\Sigma_2$) of propositional formulas whose size is polynomial w.r.t. $\Sigma$. Let $M_{\mathcal{A}}$ and $M_{\mathcal{B}}$ be the sets of facts in $\mathcal{A}$ and $\mathcal{B}$, respectively. Clearly, to determine whether $\Sigma$ has both a trivial model and a sharp model, it suffices to check whether both $M_{\mathcal{A}}\models\Sigma_1$ and $M_{\mathcal{B}}\models\Sigma_2$ are true. It is well-known that model checking propositional formulas be solved in \textsc{PTime}, which yields the desired bound. 
\end{proof} 

Now let us turn to the complexity of rewritability from $\mathsf{DED}$ to $\mathsf{ED}$ [over finite structures]. By Theorem~\ref{prop:dgd_dp_prsv}, it suffices to identify the complexity of preservation under direct products [in the finite]. By the following, we then have the desired bounds.

\begin{theorem}\label{prop:ld_union_prsv_up}
Determining whether
a given finite set of DEDs is preserved under direct products is \textsc{RE}-complete, and the finite model-theoretic version of this problem is \textsc{coRE}-complete.
\end{theorem}

\begin{proof}
(Hardness) We prove it by reducing the boolean query answering of EDs to the preservation of DED-sets. It had been proved in~\cite{BeeriV81} that [the finite model-theoretic version of] the former is [\textsc{co}]\textsc{RE}-hard. Let $\Sigma$ be a finite set of TGDs, and $q$ an boolean atomic query of the form $\exists\vec{x}\alpha(\vec{x})$. Let $P,Q$ and $R$ be nullary relation symbols without occurrence in $\Sigma$ and $q$. For each TGD $\sigma$ of the form
\begin{equation}
\varphi(\vec{x},\vec{y})\rightarrow\exists\vec{z}\psi(\vec{x},\vec{z}),
\end{equation}
let $\sigma'$ denote the DED
\begin{equation}
\varphi(\vec{x},\vec{y})\rightarrow\exists\vec{z}\psi(\vec{x},\vec{z})\vee Q.
\end{equation}
Let $\Sigma'$ be the set consists of $\sigma'$ for all TGDs $\sigma\in\Sigma$. Let $\Sigma''$ denote the set consists of the following DEDs: 
\begin{eqnarray}
\alpha(\vec{x})&\rightarrow&Q,\\
R&\rightarrow& P\vee Q.
\end{eqnarray}
Let $\Sigma_0$ denote the union of $\Sigma'$ and $\Sigma''$. To complete the proof for the hardness, it suffices to prove the following property:

\medskip
{\noindent\em Claim. $\Sigma\vDash_{[\mathrm{fin}]} q$ iff $\Sigma_0$ is preserved under direct products [in the finite].}
\medskip

(Only-if) Suppose $\Sigma\vDash_{[\mathrm{fin}]} q$. Let $\mathcal{A}$ and $\mathcal{B}$ be [finite] models of $\Sigma_0$. We need to prove that $\mathcal{A}\times\mathcal{B}$ is a model of $\Sigma$.
We first claim $Q^{\mathcal{A}}=Q^{\mathcal{B}}=true$. (Otherwise, let us assume, e.g., $Q^\mathcal{A}=false$. It is thus clear that $\mathcal{A}$ is a model of $\Sigma$. From $\Sigma\vDash_{[\mathrm{fin}]}q$ we know $\mathcal{A}\models q$, which implies $Q^\mathcal{A}=true$, a contradiction.) Consequently, we have $Q^{\mathcal{A}\times\mathcal{B}}=true$, which implies that $\mathcal{A}\times\mathcal{B}$ is indeed a model of $\Sigma_0$.

\smallskip
(If) Suppose $\Sigma\vDash_{[\mathrm{fin}]}q$ is false. That is, there is a [finite] model, say $\mathcal{A}$, of $\Sigma$ that does not satisfy $q$. Let $\tau$ be the signature of $\Sigma_0$. Let $\mathcal{A}_1$ and $\mathcal{A}_2$ be $\tau$-expansions of $\mathcal{A}$ such that 
\begin{enumerate}
\item $R^{\mathcal{A}_1}=P^{\mathcal{A}_1}=true$, $Q^{\mathcal{A}_1}=false$;
\item $R^{\mathcal{A}_2}=Q^{\mathcal{A}_2}=true$, $P^{\mathcal{A}_2}=false$.
\end{enumerate}
It is easy to see that both $\mathcal{A}_1$ and $\mathcal{A}_2$ are models of $\Sigma_0$, but $
\mathcal{A}_1\times\mathcal{A}_2$ is not. So $\Sigma_0$ is not preserved under direct products [in the finite], which completes the proof of hardness. 

\smallskip
(Membership) Similar to that in the proof of Theorem~\ref{thm:cmplx_ghom_prsv}. 
\end{proof}

For the [\textsc{co}]\textsc{RE}-membership for the rewritability of $\mathsf{ED}$ to $\mathsf{TGD}$ [over finite structures], by Theorem~\ref{thm:strict_hom_prsv}, it suffices to prove

\begin{proposition}\label{prop:stric_hom_prsv_up}
Determining whether a given finite set of EDs is preserved under both strictly-homomorphic images and preimages is in $\textsc{RE}$, and the finite model-theoretic version of this problem is in $\textsc{coRE}$. 
\end{proposition}

\begin{proof}
Similar to the proof of membership in Theorem~\ref{thm:cmplx_ghom_prsv}.
\end{proof}

Similarly, for the [\textsc{co}]\textsc{RE}-membership for the rewritability of $\mathsf{TGD}$ to $\mathsf{FGTGD}$ [over finite structures], by Theorem~\ref{thm:char_d2fgd}, it suffices to prove

\begin{proposition}\label{prop:iso_union_prsv_up}
Determining whether a given finite set of EDs is preserved under isomorphic unions is in $\textsc{RE}$, and the finite model-theoretic version of this problem is in $\textsc{coRE}$. 
\end{proposition}

\begin{proof}
Similar to the proof of membership in Theorem~\ref{thm:cmplx_ghom_prsv}.
\end{proof} 

To obtain the \textsc{2ExpTime}-completeness for the rewritability of $\mathsf{FGTGD}$ to $\mathsf{GTGD}$ [over finite structures], by Theorem~\ref{thm:char_tgd2gd}, it suffices to prove the following theorem:

\begin{theorem}\label{prop:ld_union_prsv_up}
Determining whether
a given finite set of frontier-guarded TGDs is preserved under disjoint unions [in the finite] is \textsc{2ExpTime}-complete.
\end{theorem}

\begin{proof}
(Hardness) The general idea is to reduce the boolean query answering problem for guarded TGDs to the problem mentioned in the theorem. Let $\Sigma$ be a finite set of guarded TGDs, and $q$ a boolean atomic query. It had been proved by~\citeauthor{CaliGK13}~\shortcite{CaliGK13} that determining whether $\Sigma\vDash q$ [in the finite] is $\textsc{2ExpTime}$-hard. Suppose $q$ is of the form $\exists\vec{x}\alpha(\vec{x})$. Let 
$\Sigma'$ be the set that consists of the follows rules:
\begin{eqnarray}
\label{eqn:tgd_du_com_1}\alpha(\vec{x})&\rightarrow& Q,\\
\label{eqn:tgd_du_com_2}Q\wedge E(x,y) &\rightarrow& C(y),\\
\label{eqn:tgd_du_com_3}Q\wedge E(y,z) &\rightarrow& C(y),\\
\label{eqn:tgd_du_com_4}E(x,y) \wedge E(y,z) &\rightarrow& C(y),
\end{eqnarray}
where $Q,E$ and $C$ are fresh relation symbols without occurrence in $\Sigma$ and $q$.
Let $\Sigma_0$ denote $\Sigma\cup\Sigma'$. Now we show   

\medskip
{\noindent\em Claim. $\Sigma\vDash_{[\mathrm{fin}]} q$ iff $\Sigma_0$ is preserved under disjoint unions [in the finite].} 

\medskip
To yield the hardness, it remains to prove the above claim:

\smallskip
(Only-if) Assume $\Sigma\vDash q$. Let $\mathcal{A}$ and $\mathcal{B}$ be arbitrary [finite] models of $\Sigma_0$ such that $\mathcal{A}|_X=\mathcal{B}|_X$ whenever $X=A\cap B$. It is easy to see that $Q$ must be interpreted as $true$ in both $\mathcal{A}$ and $\mathcal{B}$. Suppose $\mathcal{C}\in\{\mathcal{A},\mathcal{B}\}$. Let 
\begin{align*}
E_{-}^\mathcal{C}&=\{a:\exists b\text{ s.t. }(a,b)\in E^\mathcal{C}\},\\
E_{+}^\mathcal{C}&=\{b:\exists a\text{ s.t. }(a,b)\in E^\mathcal{C}\}.
\end{align*}
Then, we have $$E_{-}^\mathcal{C}\subseteq C^\mathcal{C}\text{ and }E_{+}^\mathcal{C}\subseteq C^\mathcal{C}.$$
We need to prove that $\mathcal{A}\cup\mathcal{B}$ is a model of $\Sigma_0$. Since $\Sigma$ is a finite set of guarded TGDs, by Proposition~\ref{prop:gd_duprv}, $\Sigma$ is preserved under disjoint unions, which implies that $\mathcal{A}\cup\mathcal{B}$ is a model of $\Sigma$. By definition, it is clear that $Q^{\mathcal{A}\cup\mathcal{B}}=true$, which implies that $\mathcal{A}\cup\mathcal{B}$ satisfies TGD~(\ref{eqn:tgd_du_com_1}). Since 
$$E_{-}^{\mathcal{A}\cup\mathcal{B}}= E_{-}^{\mathcal{A}}\cup E_{-}^{\mathcal{B}}\subseteq C^\mathcal{A}\cup C^\mathcal{B}=C^{\mathcal{A}\cup\mathcal{B}},$$
we know that $\mathcal{A}\cup\mathcal{B}$ is a model of TGD~(\ref{eqn:tgd_du_com_2}). By a similar argument, we have that $\mathcal{A}\cup\mathcal{B}$ is a model of TGD~(\ref{eqn:tgd_du_com_3}).
Suppose $a,b,c\in A\cup B$ are constants such that $(a,b)\in E^{\mathcal{A}\cup\mathcal{B}}$ and $(b,c)\in E^{\mathcal{A}\cup\mathcal{B}}$. To prove $\mathcal{A}\cup\mathcal{B}$ is a model of $\Sigma_0$, it remains to show $b\in C^{\mathcal{A}\cup\mathcal{B}}$. From $(a,b)\in E^{\mathcal{A}\cup\mathcal{B}}$, by definition we obtain either $(a,b)\in E^\mathcal{A}$ or $(a,b)\in E^\mathcal{B}$. Thus, either $b\in E_+^\mathcal{A}$ or $b\in E_+^\mathcal{B}$ is true, i.e., $b\in E_+^{\mathcal{A}\cup\mathcal{B}}$. Consequently, we have $b\in C^{\mathcal{A}\cup\mathcal{B}}$ as desired. By the arbitrariness of $\mathcal{A}$ and $\mathcal{B}$, we know that $\Sigma_0$ is preserved under disjoint unions [in the finite].

\smallskip
(If) Suppose $\Sigma\vDash_{\mathrm{fin}} q$ is false. We need to show that $\Sigma_0$ is not preserved under disjoint unions [in the finite]. Let $\mathcal{A}$ be a [finite] model of $\Sigma$ such that $\mathcal{A}\not\models q$. W.l.o.g., assume $A$ contains at least 3 distinct constants $a$, $b$ and $c$. (Otherwise, as $\Sigma$ and $\neg q$ are preserved under disjoint unions, by joining $\mathcal{A}$ with its disjoint copies, one can always obtain the desired model.)  Let $\mathcal{A}_1$ and $\mathcal{A}_2$ be isomorphic copies of $\mathcal{A}$ with $A_1\cap A_2=\{b\}$. Let $a'$ (resp., $c'$) be the copy of $a$ (resp., $c$) in $\mathcal{A}_1$ (resp., $\mathcal{A}_2$). 
Let $\tau$ be the signature of $\Sigma_0$. Let $\mathcal{A}^+_1$ and $\mathcal{A}^+_2$ be the $\tau$-expansions of $\mathcal{A}_1$ and $\mathcal{A}_2$, respectively, defined by 
\begin{enumerate}
\item $Q^{\mathcal{A}^+_1}=Q^{\mathcal{A}^+_2}=false$;
\item $E^{\mathcal{A}^+_1}=\{(a',b)\}$, $E^{\mathcal{A}^+_2}=\{(b,c')\}$, $C^{\mathcal{A}^+_1}=C^{\mathcal{A}^+_2}=\emptyset$.
\end{enumerate}
Clearly, both $\mathcal{A}^+_1$ and $\mathcal{A}^+_2$ are models of $\Sigma_0$, but $\mathcal{A}^+_1\cup\mathcal{A}^+_2$ is not. $\Sigma_0$ is thus not preserved under unions [in the finite].

\smallskip
(Membership) To prove this, we first give a brief review of the guarded negation fragment, a decidable fragment of first-order logic recently proposed in~\cite{BaranyCS2015}. Given a signature $\tau$, formulas in {\em guarded negation fragment} of first-order logic (GNFO) over $\tau$ are defined as follows:
\begin{equation}
\varphi::=\alpha\mid\exists x\varphi\mid\varphi\wedge\varphi\mid\varphi\vee\varphi\mid\alpha(\vec{x},\vec{y})\wedge\neg\varphi(\vec{x})
\end{equation}
where $\alpha$ is an atomic formula over $\tau$ and $\vec{x},\vec{y}$ are tuples of variables. Note that constant symbols are allowed to occur in formulas. By Proposition 7.3 of~\cite{BaranyCS2015}, there is a polynomial-time transformation to eliminate constant symbols, which preserves logical equivalence. Thus, the complexity of satisfiability of the language here remains the same as the constant-free GNFO, which is \textsc{2ExpTime}-complete; in addition, the finite model property remains true.

To yield the desired $\textsc{2ExpTime}$-membership, it suffices to prove that determining whether a given GNFO-sentences is preserved under disjoint unions [in the finite] is in $\textsc{2ExpTime}$. The general idea is as follows: To determine whether a given GNFO-sentence $\varphi$ is preserved under disjoint unions [in the finite], we construct a GNFO-sentence $\tilde{\varphi}$ such that $\varphi$ is preserved under disjoint unions [in the finite] iff $\tilde{\varphi}$ is unsatisfiable [over finite structures]. Since the [finite] satisfiability problem of GNFO is \textsc{2ExpTime}-complete~\cite{BaranyCS2015}, we then have the desired membership. 

To implement this idea, we first introduce some fresh unary relation symbols $D_{A}$ and $D_{B}$. Let $\tau$ be the signature of $\varphi$ and let $\tau_0=\tau\cup\{D_{A},D_{B}\}$. Let $\vartheta$ denote $\forall x(D_A(x)\vee D_B(x))$ which can be equivalently rewritten as the GNFO-sentence 
$$
\exists y(y=y\wedge \neg\exists x(x=x\wedge\neg (D_{A}(x)\vee D_{B}(x)))).
$$
Fix $\mathcal{C}$ as a $\tau_0$-structure. Set $A=D_A^\mathcal{C}$ and $B=D_B^\mathcal{C}$. Let $\mathcal{A}$ (resp., $\mathcal{B}$) be the restriction of $\mathcal{C}|_A$ (resp., $\mathcal{C}|_B$) to $\upsilon$. It is not difficult to see that (i) if $\mathcal{C}$ is a model of $\vartheta$ then $\mathcal{A}\cup\mathcal{B}$ is the restriction of $\mathcal{C}$ to $\tau$, and (ii) $\mathcal{A}|_{X}=\mathcal{B}|_{X}$ where $X=A\cap B$.

Take $\circ\in\{A,B\}$. Define a transformation $\pi_\circ$ as follows: 
\begin{equation*}
\pi_\circ(\varphi)\!=\!\left\{
\begin{aligned}
&\alpha(\vec{x})\wedge D_\circ(\vec{x})\quad&\text{if }&\varphi=\alpha(\vec{x}), \alpha\text{ is atomic}\!\!\!\!\!\!&,\\
&\neg \pi_\circ(\psi)\!\!\!&\text{if }&\varphi=\neg\psi\!\!\!\!\!\!\!&,\\
&\exists x\,\pi_\circ(\psi)\!\!\!&\text{if }&\varphi=\exists x\,\psi&,\\
&\pi_\circ(\psi)\wedge \pi_\circ(\chi)\!\!\!&\text{if }&\varphi=\psi\wedge\psi\!\!\!\!\!\!\!&,\\
&\pi_\circ(\psi)\vee \pi_\circ(\chi)\!\!\!&\text{if }&\varphi=\psi\vee\chi&,
\end{aligned}
\right.
\end{equation*}
where $D_\circ(\vec{x})$ denotes the formula $D_\circ(x_1)\wedge\cdots\wedge D_\circ(x_k)$ if $\vec{x}=x_1\cdots x_k$.
It is not difficult to show that $\mathcal{C}\models \pi_\circ(\varphi)$ iff $\varphi$ is satisfied in the substructure of $\mathcal{C}$ which is induced by the set $D_\circ^{\mathcal{C}}$. Let $
\tilde{\varphi}$ denote the sentence $$\vartheta\wedge\pi_A(\varphi)\wedge\pi_B(\varphi)\wedge\neg\varphi.$$ Clearly, $\tilde{\varphi}$ is equivalent to a GNFO-sentence. It is also not difficult to check that $\tilde{\varphi}$ satisfies the desired property, i.e., $\varphi$ is preserved under disjoint unions [in the finite] iff $\tilde{\varphi}$ is unsatisfiable [over finite structures], which completes the proof. 
\end{proof}

To prove the \textsc{PSpace}-hardnes for the rewritability of $\mathsf{GTGD}$ to $\mathsf{LTGD}$ [over finite structures], by Theorems~\ref{thm:char_tgd2ld_finite} and \ref{thm:char_tgd2ld}, it suffices to prove the following proposition:

\begin{proposition}\label{prop:ld_union_prsv_up}
Determining whether a given finite set of GTGDs is preserved under unions [in the finite] is \textsc{PSpace}-hard.
\end{proposition}

\begin{proof}
We prove this by reducing the boolean query answering problem of linear TGDs to our desired problem. Given a finite set $\Sigma$ of linear TGDs and a boolean atomic query $q$ of the form $\exists\vec{x}\alpha(\vec{x})$, it is implicit by Theorem 4.4 of~\cite{GottlobP03} that deciding whether $\Sigma\vDash_{[\mathrm{fin}]} q$ is $\textsc{PSpace}$-hard. Let 
$\Sigma'$ be a set consisting of following TGDs:
\begin{eqnarray}
\alpha(\vec{x})&\rightarrow& Q,\\
Q\wedge R(x) &\rightarrow& T(x),\\
Q\wedge S(x) &\rightarrow& T(x),\\
R(x) \wedge S(x) &\rightarrow& T(x),
\end{eqnarray}
where $R,S,T$ and $Q$ are fresh relation symbols without occurrence in $\Sigma$ and $q$.
Let $\Sigma_0$ denote $\Sigma\cup\Sigma'$. To yield the desire \textsc{PSpace}-hardness, it suffices to prove the following property:  

\medskip
{\noindent\em Claim. $\Sigma\vDash_{[\mathrm{fin}]} q$ iff $\Sigma_0$ is preserved under unions [in the finite].} 

\medskip
(Only-if) Suppose $\Sigma\vDash_{[\mathrm{fin}]} q$. Let $\mathcal{A}$ and $\mathcal{B}$ be [finite] models of $\Sigma_0$. It is easy to see that $Q$ must be interpreted as $true$ in both $\mathcal{A}$ and $\mathcal{B}$. Consequently, for $\mathcal{C}\in\{\mathcal{A},\mathcal{B}\}$ it holds that $$R^\mathcal{C}\cup S^\mathcal{C}\subseteq T^\mathcal{C}.$$
From this we thus have 
$$R^{\mathcal{A}\cup\mathcal{B}}\cup S^{\mathcal{A}\cup\mathcal{B}}=R^\mathcal{A}\cup R^\mathcal{B}\cup S^\mathcal{A}\cup S^\mathcal{B}\subseteq T^\mathcal{A}\cup T^\mathcal{B}=T^{\mathcal{A}\cup\mathcal{B}}.$$
We thus conclude that $\mathcal{A}\cup\mathcal{B}$ is a model of $\Sigma'$. By Theorem~\ref{thm:char_tgd2ld}, we know that $\Sigma$ is preserved under unions, which implies that $\mathcal{A}\cup\mathcal{B}$ is a model of $\Sigma$. Hence, $\mathcal{A}\cup\mathcal{B}$ is also a model of $\Sigma_0$, which implies that $\Sigma_0$ is preserved under unions [in the finite].

\smallskip
(If) Suppose $\Sigma\vDash_{[\mathrm{fin}]} q$ is not true. It suffices to show that $\Sigma_0$ is not preserved under unions [in the finite]. Let $\mathcal{A}$ be a [finite] model of $\Sigma$ which does not satisfy $q$. Let $\tau$ be the signature of $\Sigma_0$. Let $\mathcal{A}_1$ and $\mathcal{A}_2$ be $\tau$-expansions of $\mathcal{A}$ which are defined by
\begin{enumerate}
\item $Q^{\mathcal{A}_1}=Q^{\mathcal{A}_2}=false$;
\item $R^{\mathcal{A}_1}=S^{\mathcal{A}_2}=\{a\}$, $R^{\mathcal{A}_2}=S^{\mathcal{A}_1}=T^{\mathcal{A}_1}=T^{\mathcal{A}_2}=\emptyset$, where $a$ is a constant in $A$.
\end{enumerate}
It is clear that both $\mathcal{A}_1$ and $\mathcal{A}_2$ are models of $\Sigma_0$, but $\mathcal{A}_1\cup\mathcal{A}_2$ is not. Thus, $\Sigma_q$ is not preserved under unions [in the finite].
\end{proof}

}

\end{document}